\documentclass[lettersize,journal]{IEEEtran}
\usepackage{amsmath,amsfonts}
\usepackage{textcomp}
\usepackage{stfloats}
\usepackage{url}
\usepackage{verbatim}
\usepackage{graphicx}
\usepackage{cite}

\usepackage{xcolor}
\usepackage{multicol}

\definecolor{mycitecolor}{RGB}{71, 191, 38}
\definecolor{mylinkcolor}{RGB}{40, 115, 201}

\usepackage[bookmarks=true, colorlinks, citecolor=mycitecolor,linkcolor=mylinkcolor,urlcolor=mycitecolor ]{hyperref}

\usepackage{amsthm}
\newtheorem{theorem}{\bf Theorem}
\newtheorem{definition}{\bf Definition}
\newtheorem{lemma}{\bf Lemma}
\newtheorem{remark}{\bf Remark}
\newtheorem{assumption}{Assumption}
\usepackage{multicol}
\usepackage{caption}
\usepackage{subcaption}
\usepackage{float}

\usepackage{algpseudocode}

\usepackage[ruled,vlined,noend]{algorithm2e}

\begin{document}

\title{Strategic Decision-Making in Multi-Agent Domains: A Weighted Constrained Potential Dynamic Game Approach}

\author{Maulik Bhatt$^{1}$, Yixuan Jia$^{2}$ and Negar Mehr$^{3}$
\thanks{$^1$Maulik Bhatt ({\tt\small maulikbhatt@berkeley.edu }) and $^{2}$Negar Mehr ({\tt\small negar@berkeley.edu }) are with the Department of Mechanical Engineering, University of California at Berkeley, CA, USA. \\
$^3$Yixuan Jia ({\tt\small yixuany@mit.edu }) is with LIDS, Massachusetts Institute of Technology, MA, USA }.
\thanks{This work is supported by the National Science Foundation, under grants ECCS-2438314 CAREER Award, CNS-2423130, and CCF-2423131.}
}

\maketitle

\begin{abstract}
In interactive multi-agent settings, decision-making and planning are challenging mainly due to the agents' interconnected objectives. Dynamic game theory offers a formal framework for analyzing such intricacies. Yet, solving constrained dynamic games and determining the interaction outcome in the form of generalized Nash Equilibria (GNE) pose computational challenges due to the need for solving constrained coupled optimal control problems. In this paper, we address this challenge by proposing to leverage the special structure of many real-world multi-agent interactions. More specifically, our key idea is to leverage constrained dynamic potential games, which are games for which GNE can be found by solving a single constrained optimal control problem associated with minimizing the potential function. We argue that constrained dynamic potential games can effectively facilitate interactive decision-making in many multi-agent interactions. We will identify structures in realistic multi-agent interactive scenarios that can be transformed into weighted constrained potential dynamic games (WCPDGs). We will show that the GNE of the resulting WCPDG can be obtained by solving a single constrained optimal control problem. We will demonstrate the effectiveness of the proposed method through various simulation studies and show that we achieve significant improvements in solve time compared to state-of-the-art game solvers. We further provide experimental validation of our proposed method in a navigation setup involving two quadrotors carrying a rigid object while avoiding collisions with two humans.
\end{abstract}


\section{Introduction}\label{sec:intro}
\IEEEPARstart{N}{umerous} robotic applications, such as autonomous driving, crowd-robot navigation, and delivery robots, include scenarios that require multi-agent interactions. In these situations, a robot is tasked with engaging with either human individuals or other robots present in the surrounding environment. Making decisions and creating plans in these domains is challenging as agents in multi-agent systems may have different goals or objectives. This requires agents to develop strategies that account for how other agents will react to their actions, i.e., each agent needs to reason about the likely reactions of other agents in their decision-making. Furthermore, agents may need to satisfy some task constraints, such as collision avoidance and goal constraints.
Such reasoning may be computationally demanding, requiring agents to perform joint prediction and planning.

\begin{figure*}[t]
    \centering
    \vspace{0.25cm}
    \includegraphics[scale=0.63]{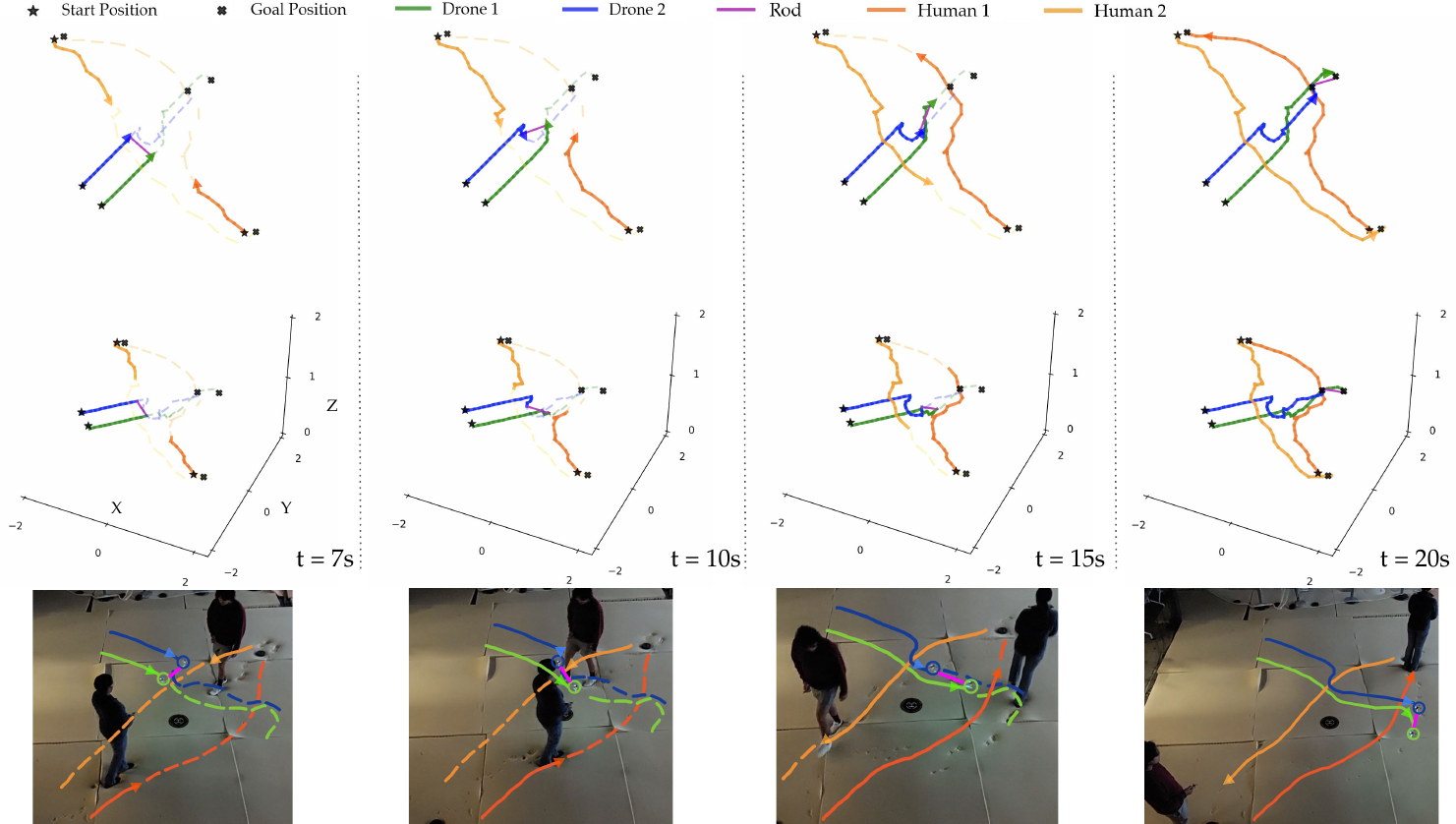}
    \caption{Visualizations and the corresponding snapshots of the hardware experiment. 
    The two quadrotors and two humans need to reach their designated goal positions while avoiding collision with others. The two quadrotors are carrying a rod. In this case, each quadrotor needs to account for the motions of humans as well as the other quadrotor, as they share the constraint imposed by the rod.
    As shown in the figure, with our algorithm, the two quadrotors are able to avoid nearby humans by changing their orientations elegantly while carrying the rod. (Refer to Appendix~\ref{sec: exempt} for the approval details for conducting experiments with human subjects.) {We have also included a supplementary video of the experiment which will be available at \href{http://ieeexplore.ieee.org}{http://ieeexplore.ieee.org}.}}
    \label{fig:hardware_experiment}
    \vspace{-0.55cm}
\end{figure*}

Dynamic Game theory is a powerful tool for analyzing and understanding multi-agent interactions because it provides a formal framework for studying strategic decision-making in situations where an agent's objective may depend not only on its own choices but also on those of the other agents. 
In particular, dynamic noncooperative game theory \cite{bacsar1998dynamic} provides a formalism for sequential decision-making involving multiple strategic decision-makers. It has been shown that Nash equilibria of dynamic games can model the interaction outcome among multiple strategic decision makers \cite{isaacs1999differential}.
Generalized Nash Equilibria (GNE) in dynamic games consists of sets of strategies (formally defined later) for the agents involved in the game such that no agent can benefit more by deviating from it while satisfying the safety constraints. Mathematically, finding the GNE of a constrained dynamic game can be formulated as a set of coupled constrained optimal control problems.
However, real-world robot planning problems have non-linear dynamics, costs, and constraints. Therefore, solving the coupled constrained non-linear optimal control problems for computing Nash equilibria is a highly challenging task. This computational bottleneck inhibits most of the current game theoretic motion planning methods from being used in real-time.

In this paper, we tackle this challenge and provide a framework for tractably finding the equilibria of multi-agent interactions that involve strategic decision-makers. Our key idea is that, by utilizing ideas from \emph{potential games} \cite{monderer1996potential}, we can significantly simplify the problem of computing the GNE of general-sum dynamic games. Potential games are a class of games for which a potential function exists such that the Nash equilibria of the original game can be computed by optimizing the potential function. Potential games often appear in economics and engineering applications, where multiple agents share a common resource (a raw material, a communication link, a transportation link, an electrical transmission line) or limitations. \emph{We argue that in many multi-agent interactions, agents either share a space or a task}. Thus, we propose that potential games, particularly in the form of constrained dynamic potential games, can be effectively utilized to facilitate interactive decision-making. 

Using constrained potential games, we reduce the coupled constrained optimal control problems of multi-agent motion planning into a well-studied single constrained optimal control problem. With this simplification, we can significantly reduce the computational complexity of multi-agent motion planning. Using any off-the-shelf optimizer, we can efficiently solve the resulting single optimal control problem, making our algorithm faster and more employable in real-time. However, previous works that utilized potential games for multi-agent trajectory planning \cite{kavuncu2021potential, bhatt2022efficient} could only capture symmetric inter-agent costs or couplings. These approaches limit the types of cost structures that can be captured for multi-agent interactions.

In this work, we go beyond symmetric couplings and identify various practical cost structures under which multi-agent decision-making and planning turn into weighted constrained potential dynamic games (WCPDGs). We show how the GNE of the underlying dynamic game can be found reliably and tractably for WCPDGs by solving a single constrained optimal control problem. We will show that under a set of mild assumptions, dyadic interactions, i.e., interactions between two agents, are always WCPDGs. We will then consider multi-agent settings with more than two agents and identify a set of realistic cost structures under which interactions become WCPDGs. We will discuss how the resulting single constrained optimal control problem can be solved using off-the-shelf solvers such as ALTRO~\cite{howell2019altro}, do-mpc \cite{LUCIA201751} and iterative-LQR \cite{li2004iterative} for efficient multi-agent trajectory planning with nonlinear state and action constraints. Through several simulation studies, we will show that our proposed method provides an efficient solution for finding equilibria. To demonstrate the practical relevance of our proposed method, we also showcase a hardware experiment. We will further compare the solve time of our method with the state-of-the-art game solvers and show that our method is significantly faster than the state-of-the-art. In summary, the main contributions of this paper are as follows:
\begin{itemize}
    \item We find practical cost structures under which the problem of multi-agent motion planning problem is a WCPDG.
    \item We demonstrate how the GNE of the original game can be computed by optimizing the potential function of the game. This reduces the computational complexity of multi-agent motion planning, making the algorithm scalable, faster, and in real time. 
    \item We show the benefits of our proposed method in a number of simulations, showing efficient and intuitive solutions while outperforming state-of-the-art game solvers in terms of computational speed.
    \item We also validate our method in an experiment setting where two drones carried an object while moving around humans.
\end{itemize}
The organization of this paper is as follows.
Section \ref{sec:relatedwork} provides a literature review on multi-agent motion planning. Notations and problem formulation are provided in Section \ref{sec:notation}. In Section \ref{sec:dynamic-games}, we present important results on WCPDG for game theoretic motion planning. Then, we provide realistic cost structures for dyadic and multi-agent interactions under which the resulting game will be a WCPDG in Sections \ref{sec:dyadic-interactions} and \ref{sec:multi-agent-interactions}, respectively. Simulations and experiments with the analysis of the resulting performance are incorporated in Section \ref{sec:simulations} and \ref{sec:experiment}, respectively. Finally, Section \ref{sec:Conclusion} concludes this paper with a discussion on future directions.

\begin{figure*}[th]
    \centering
    \includegraphics[scale = 0.26]{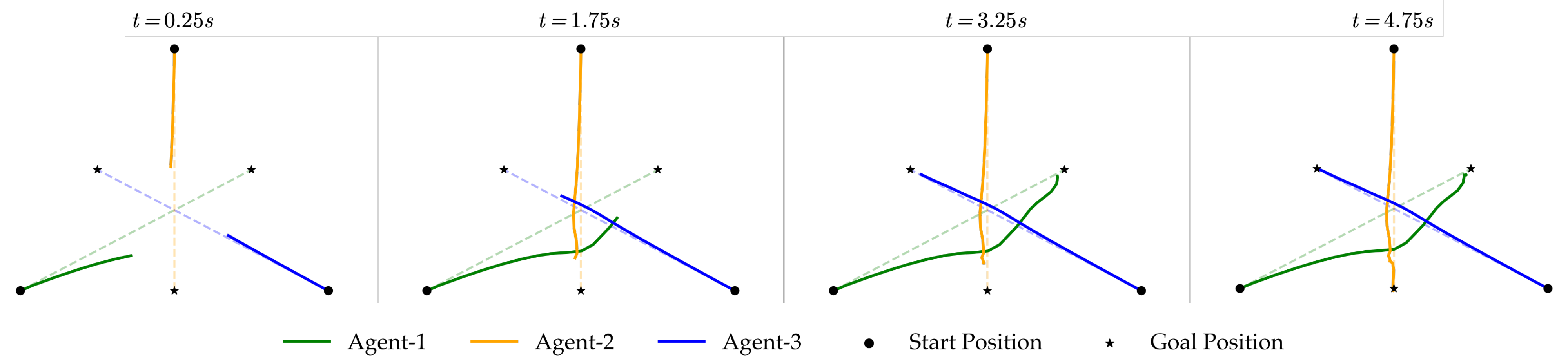}
    \caption{The snapshots of the trajectories found by our algorithm when three unicycle agents interact with each other and exchange their positions diagonally over a time interval of 5s. The cost functions are highly asymmetric, and Agent 1 incurs more costs for getting close to other agents. Agents move from their start positions to their respective goal positions. Dashed lines represent the nominal path for each agent. Because of the asymmetric nature of the interaction, Agent 1 yields to other agents and deviates more from its nominal path to move towards its goal.}
    \label{fig:three_unicycle_potential}
\end{figure*}

\section{Related Work}\label{sec:relatedwork}
\subsection{Multi-agent Motion Planning}
Multi-agent motion planning is generally classified into two categories, cooperative and noncooperative, based on how agents interact with each other in the environment. As the name suggests, cooperative motion planning refers to planning in settings where agents are all part of the same team. Various cooperative motion planning algorithms have been used, such as \cite{yu2016optimal, tang2018hold, augugliaro2012generation, desaraju2012decentralized, toumieh2022decentralized}. Multi-agent motion planning using buffered voronoi cells was introduced in \cite{zhou2017fast}. There have been various learning-based planners for multi-agent motion planning that utilize past trajectory data to improve motion planning in the future. A few representative examples of such methods are \cite{lv2019path, qie2019joint, liu2020mapper, semnani2020multi, vinod2022safe}. However, learning-based methods generally suffer from non-stationarity in the environment because of simultaneous learning~\cite{papoudakis2019dealing,de2020independent}.
Furthermore, learning-based methods generally focus on the path-planning aspect and fail to consider the interactive nature of multi-agent motion planning problems. The work in \cite{kretzschmar2016socially} presents an inverse reinforcement learning-based framework for cooperative crowd navigation scenarios. This work explicitly models the interactions between the robot and humans but does not provide a theoretical guarantee of finding equilibrium solutions that abide safety constraints as well. It should be noted that most of the aforementioned methods are cooperative in nature. However, many real-world motion planning scenarios contain agents with individual goals and objectives. Previous works fail to work in such settings. Many game theoretical approaches exist for such noncooperative motion planning problems, which are discussed in the following section.
\subsection{Game-Theoretic Multi-Agent Interactions}
Game theory has proven effective for multi-agent systems where agents are individual decision-makers \cite{parsons2002game}. Due to the interactive nature of multi-agent planning, where agents have to interact with each other and take into account the likely reactions of other agents, game theory naturally becomes a robust framework for reasoning about such interdependencies. Seeking Nash equilibria of the underlying game in the interactive multi-agent planning domain generates highly intuitive and interactive trajectories. 

A Stackelberg equilibrium was initially considered in two-player games for modeling the mutual coupling between two agents ~\cite{sadigh2016planning}. In the context of autonomous driving, a Markovian Stackelberg strategy was developed using hierarchical planning in~\cite{fisac2019hierarchical}. However, the proposed dynamic programming approach suffers from the curse of dimensionality. Moreover, Stackelberg equilibria are not suited to multi-agent setups. Stackelberg equilibria commonly involve a leader and a follower; hence, they assume a form of asymmetry among the agents in that one agent reacts to another. A human-like motion planner based on game theoretic decision-making was proposed in \cite{turnwald2019human}. However, to compute Nash equilibria, they generated a random set of sample trajectories and chose the best samples using Pareto optimality, which doesn't necessarily recover Nash equilibria. 

Iterative best-response methods were developed in~\cite{spica2020real, wang2019game, wang2019game-2} in the context of autonomous racing. However, such an approach for finding Nash equilibria can be slow because computing the best response for each agent generally involves solving a non-convex optimization problem. Therefore, authors in \cite{miller2022multi} proposed a Lexicograhpic-based method that helps find approximate Nash equilibria using a linear version of the iterated best response. However, this method does not consider hard constraints on states and actions.

Iterative linear quadratic approximations for real-time game theoretic motion planning were introduced in \cite{fridovich2020efficient}. A belief space-based motion planning method was proposed in \cite{schwarting2021stochastic}. Augmented Lagrangian-based game theoretic solver (ALGAMES) was proposed in \cite{le2022algames}. In \cite{laine2023computation}, the authors proposed a sequential linear-quadratic technique game-based technique to compute generalized Nash equilibria. Seeking equilibria of stochastic variants of general-sum dynamic games was considered in~\cite{wang2020game} and~\cite{mehr2023maximum}. A fault-tolerant receding horizon game-theoretic motion planner that leverages inter-agent communication with intention hypothesis likelihood for multiple highly interactive robots was introduced in \cite{chahine2023intention}.
However, these methods are not easily scalable. They struggle to work efficiently in real-time because they attempt to solve complex coupled optimal control problems, which become more difficult as the number of agents increases. Additionally, these methods can face convergence issues as the complexity of the game increases with more agents.


\subsection{Potential Games}
Potential games have been extensively studied and analyzed, particularly for static games~\cite{rosenthal1973class,monderer1996potential}. Due to their simplicity and ease of handling, continuous-time versions of potential games called Potential Differential Games were introduced in \cite{fonseca2018potential}. Discrete-time versions of potential differential games are called potential dynamic games. Potential dynamic games and potential differential games have found practical implementations in diverse applications, including resource allocation, cooperative control, power networks, and communication networks, thanks to their simplicity and tractability~\cite{arslan2007autonomous, marden2009cooperative, zazo2014control, zazo2016dynamic}. 
Dynamic potential games were recently used in interactive multi-agent motion planning \cite{kavuncu2021potential, bhatt2022efficient, williams2023distributed}. Furthermore, dynamic potential games were utilized in autonomous racing in \cite{jia2023rapid}. However, the type of objectives these methods can capture are very restrictive as they require symmetry in the way agents assign costs to each other. 

\section{Notation and Preliminaries}\label{sec:notation}
We consider the problem of multi-agent trajectory planning in a non-cooperative setting. We model agents' interactions as a constrained dynamic game. Let $\mathcal{N} = \{1,2,\ldots, N\}$ be the set of agents' indices, where $N$ is the total number of agents. The state space of each agent $i \in \mathcal{N}$ is denoted by $\mathcal{X}^i \in \mathbb{R}^{n_i} $ where $n_i$ is the dimension of the state space of agent $i$. Let $\mathcal{X} = \mathcal{X}^1\times\cdots\times\mathcal{X}^N \in \mathbb{R}^n$ denote the set of states of all agents where $n = \sum_{i\in\mathcal{N}}n_i$ is the dimension of the state space of the overall system. We assume that the game is played over a finite horizon of time $T$. Let $x_k = \left(x^{1^\top}_k,\ldots,x^{N^\top}_k\right)^\top$ denote the state vector of the game at time $k \in \{0,1,\ldots,T\}$ with $x^i_k \in \mathcal{X}^i$ being the state vector of agent $i$ at time $k$. Furthermore, let $x^i=\{x^i_k\}_{k\in\{0,1,\ldots,T\}}$ be the collection of states of agent $i$ at all times. Likewise, let $x=\{x_k\}_{k\in\{0,1,\ldots, T\}}$ be the collection of the overall system state vector at all times.

For each agent, $i \in \mathcal{N}$, let $\mathcal{U}^i \subseteq \mathbb{R}^{m_i}$ denote the action space of agent $i$ where $m_i$ is the dimension of its action space.
The action space of the overall game is denoted by $\mathcal{U} = \mathcal{U}^1\times\cdots\times\mathcal{U}^N \subseteq \mathbb{R}^{m}$ where $m = \sum_{i\in\mathcal{N}}m_i$. The vector of agents' actions at time $k$ is denoted by $ u_k := \left(u^{1^\top}_k,\ldots,u^{N^\top}_k\right)^\top \in \mathcal{U}$ where $u^i_k \in \mathcal{U}^i_k$ is the vector of agent $i$'s actions at time $k$. Similar to before, we define $u^i=\{u^i_k\}_{k\in\{0,1,\ldots, T-1\}}$ and $u=\{u_k\}_{k\in\{0,1,\ldots, T-1\}}$ to be the collection of action vectors for agent $i$ and the overall system respectively over the entire horizon of time. Furthermore, for every agent $i \in \mathcal{N}$, we define $u^{-i}_k := \left(u^{1^\top}_k,\ldots,u^{{i-1}^\top}_k,u^{{i+1}^\top}_k,\ldots,u^{N^\top}_k\right)^\top \in \mathcal{U}^{-i}$ as the vector of all agents' actions at time $k$ except agent $i$, where $\mathcal{U}^{-i} := \Pi_{j\neq i}\mathcal{U}^j$. We also have $u^{-i}=\{u^{-i}_k\}_{k\in\{0,1,\ldots,T-1\}}$. With a slight abuse of notation, we can write $u_k = \left(u^{i^\top}_k,u^{{-i}^\top}_k\right)^\top$. 

We assume that each agent's dynamics depend only on its own states and actions via:
\begin{equation}
    x^i_{k+1}=f^i(x^i_k,u^i_k).
\end{equation}
Note that this is a valid assumption as most real-world robotic interactions contain robots that have separable dynamics, and the couplings among the agents normally arise from the couplings inherent in agents' objectives. The overall system dynamics for the game are defined by $f: \mathcal{X}\times\mathcal{U}\rightarrow \mathcal{X}$,
\begin{equation}
    x_{k+1} = f(x_k,u_k),\label{sys_dynamics}
\end{equation}where $f = \left(f^{1^\top},\ldots,f^{N^\top}\right)^\top$. It should be noted that we have assumed time-invariant dynamics for notational simplicity, but this formulation can easily be extended to time-varying dynamics as well.

In order to account for constraints, we assume that there is an inequality constraint function $g:\mathcal{X}\times\mathcal{U}\rightarrow \mathbb{R}^c$ where $c$ is the number of constraints. The inequality constraints are defined as
\begin{equation}\label{constraints}
    g(x_k,u_k) \leq 0,
\end{equation}
where the inequality is interpreted elementwise. For the terminal time-step, the inequality constrained is only a function of the terminal state, $g(x_T)\leq0$.
It is to be noted that any equality constraint can also be included as a combination of two inequality constraints. Also, we assume that constraints that capture coupling between agents are shared among the agents. Furthermore, it should be noted that we have assumed time-invariant constraints for notional simplicity, but this formulation can easily be extended to time-varying constraints as well. Let $\mathcal{C}_k$ at each time step $k$ be the set of all the feasible states $x_k$ and actions $u_k$. We can define $\mathcal{C}_0 = \mathcal{X}\times\mathcal{U} \cap \{(x_0,u_0): g(x_0,u_0)\leq 0\}$ and $\mathcal{C}_k = \{\{\mathcal{X} \cap \{x_k \mid x_k = f(x_{k-1},u_{k-1})\} \}\times\mathcal{U}\}\cap \{(x_k,u_k): g(x_k,u_k)\leq 0\}$ for $k\in\{1,\ldots,T-1\}$. Furthermore, at the terminal time step $T$, no action is being taken, and hence, the constraint function will only be a function of states, and the corresponding constraint set will be $\mathcal{C}_{T} = \mathcal{X} \cap \{x_{T}: g(x_{T}) \leq 0 \}$.

We assume that each agent $i$ seeks to minimize a cost function $J^i:\mathcal{X}\times\mathcal{U}^i \rightarrow \mathbb{R}$ that can depend on the full state of the system and the agent's actions itself 
\begin{equation}
    J^i(x,u) = L^i_T(x_T) + \sum_{k=0}^{T-1}L^i_k(x_k,u^i_k),
\end{equation}
where $L^i_T(\cdot)$ and $L^i_k(\cdot)$ are the terminal and running costs of agent $i$ respectively. The decomposition of cost into running and terminal components offers a clear interpretation of trade-offs in decision-making. The running cost reflects immediate consequences, aiding consideration of short-term impacts. The terminal cost captures long-term objectives and provides an intuitive measure of the final outcome. Also, it should be noted that for each agent $i$, we consider costs that depend only on the actions of agent $i$ and the overall system state because we consider agents with individual objectives like reaching a goal, using minimal control efforts while avoiding collision with others. Therefore, the costs depend on the overall system state that includes other agents' states, but each agent does not care about the control effort of others, so the costs depend only on its own controls.

The goal of each agent is to choose its control actions such that its cost function is minimized while satisfying constraints. The policies through which the control actions are chosen at each time instant are called agents' strategies. We denote the strategy of each agent at time $k$ by $\gamma^i_k \in \Gamma^i$ where $\Gamma^i$ is the strategy space of agent $i$. The strategy of each agent depends on the information available to the agent at that time. Let $\eta^i_k \subseteq \{x,u\}$ be the information gained and recalled by agent $i$ at time $k$. Depending on the information structure, each agent chooses actions according to their policy $\gamma^i_k(\eta^i_k) = u^i_k$. We consider open-loop strategies that are only a function of the system's initial state and time. Therefore, we have the following relation, $\gamma^i(x_0,k):= u^i_k$. 

We would like to point out that in this paper, we focus on finding open-loop equilibria of the dynamic game, which implies that at equilibrium, the entire trajectory of each agent is the best response to the trajectories of all the other agents for a given initial state of the system, i.e., the control action of every agent is only a function of time step $k$ and the initial state. We will repeatedly solve for open-loop Nash equilibria in a receding-horizon fashion to adapt to new information obtained over time and mimic a feedback policy. We acknowledge that for general dynamic games, due to the difference between the information structure of open-loop and feedback Nash equilibria, the two concepts may result in different behaviors \cite{li2023cost}. However, for many trajectory planning purposes, a receding-horizon implementation of open-loop equilibria results in reasonable approximations~\cite{le2022algames}. Furthermore, we would like to point out that this formulation requires the assumption that all agents have knowledge of the dynamics and objectives of other agents. While knowledge of the dynamics of others is a reasonable assumption, several recent works such as \cite{mehr2023maximum} have shown that such estimates of objectives can be obtained through inverse game theoretic learning methods.

We denote the combined strategy of all agents at each time $k$ by $\gamma_k:= \left(\gamma^{1^\top}_k,\ldots,\gamma^{N^\top}_k \right)^\top \in \Gamma$ where $\Gamma:= \Gamma^1\times\cdots\times\Gamma^N$ is the strategy space of the entire game at time $k$. Similarly, we define $\gamma^i:= \{\gamma^i_k\}_{k\in\{0,1,\ldots,T-1\}}$ and $\gamma:=\{\gamma_k\}_{k\in\{0,1,\ldots,T-1\}}$. Moreover, same as before, for each agent $i \in \mathcal{N}$, we denote $\gamma_k = \left(\gamma^{i^\top}_k,\gamma^{{-i}^\top}_k\right)^\top$ where $\gamma^{-i}_k = \left(\gamma^{1^\top}_k,\ldots,\gamma^{{i-1}^\top}_k,\gamma^{{i+1}^\top}_k,\ldots,\gamma^{N^\top}_k\right)^\top$ is the tuple containing strategies of all agents except agent $i$. We further define $\gamma^{-i}:=\{\gamma^{-i}_k\}_{k\in\{0,1,\ldots,T-1\}}$ to be the strategy of all agents except agent $i$ over the entire horizon of time. Note that a given strategy $\gamma$ identifies a unique set of actions for all time instants, $\{u_k\}_{k\in\{0,1,\ldots, T-1\}}$, and; therefore, there is an equivalence between the two $\left(\gamma\equiv\{u_k\}_{k\in\{0,1,\ldots, T-1\}}\right)$. Hence, for simplicity, we use strategies and actions interchangeably from now on. Furthermore, in the case of open-loop strategies, actions at all time steps can be determined by the system's initial state, and a given strategy will, in turn, determine the states of the system at each time step. Therefore, the state $x_k$ can be written as a function of strategies, initial states, and the time step, i.e., $x_k \equiv x_k(\gamma,x_0)$. We do not denote this dependence when it is obvious from the context, but we denote it explicitly when required. Furthermore, it should be noted that not all strategies will result in trajectories that satisfy all the constraints. Therefore, we formally define feasible strategies as follows.
\begin{definition}
    Given an initial state $x_0 \in \mathcal{C}_0$, a strategy $\gamma = \left( {\gamma^1},\ldots,{\gamma^N} \right) \in \Gamma$ is a feasible strategy if the system trajectory generated by using control inputs obtain through $\gamma$ satisfies
    \begin{align}
        & \left(x_k,{\gamma}(x_0,k)\right) \in \mathcal{C}_k, k \in \{0,\ldots,T-1\}, x_{T} \in \mathcal{C}_{T}.
    \end{align}
    Furthermore, we call such a system trajectory $\{x,u\}$ to be a feasible system trajectory.
\end{definition}
Intuitively, a strategy is called a feasible strategy if the trajectory generated by applying the strategy to the system satisfies the constraints. 

Given an initial state of the system $x_0$, we represent the resulting dynamic game in a compact form as $\mathcal{G}:= \left( \mathcal{N}, \{\Gamma_i\}_{i\in\mathcal{N}}, \{J_i\}_{i\in\mathcal{N}},\{\mathcal{C}_k\}_{k\in\{0,\ldots, T\}},f, x_0 \right)$. We seek the open-loop GNE (OLGNE) of the dynamic game underlying multi-agent interactions. The OLGNE of the game is defined as:
\begin{definition}\label{gen_NE}
An open-loop generalized Nash equilibrium (OLGNE) of a game $\mathcal{G}:= \left( \mathcal{N}, \{\Gamma_i\}_{i\in\mathcal{N}}, \{J_i\}_{i\in\mathcal{N}},\{\mathcal{C}_k\}_{k\in\{0,\ldots, T\}},f,x_0 \right)$ is a feasible strategy $\gamma^{*} = \left( {\gamma^1}^*,\ldots,{\gamma^N}^* \right) \in \Gamma$ for which the following holds for each agent $i \in \mathcal{N}$ and for all feasible strategies $(\gamma^i,{\gamma^{-i}}^*) \in \Gamma$:
\begin{align}\label{eq:Nash-definition}
    & J^i(x_0,\gamma^{*}) \leq   J^i(x_0,{\gamma^{i}},{\gamma^{-i}}^*) 
\end{align}
\end{definition}
Intuitively, no agent can decrease their cost function at equilibrium by unilaterally changing their strategies to any other feasible strategy. In generalized Nash equilibria, the strategies of the agents are further coupled due to the joint constraints of the agents. It's important to note that solving~\eqref{eq:Nash-definition} involves solving a set of $N$ coupled constrained optimal control problems, which are typically difficult to solve efficiently for robotic systems with large numbers of agents with general nonlinear dynamics, cost and constraint functions.

\section{Constrained Potential Dynamic Games}\label{sec:dynamic-games}
In this section, we focus on a particular class of games called constrained potential dynamic games in which the difference in costs of each agent can be captured by the difference in a function called the potential function. A potential function assigns a value to each possible outcome of the game, such that any change in the strategy of one agent results in a change in the potential function. This property makes it relatively easy to analyze the agents' behavior and predict the game's outcome. In particular, in the case of constrained potential dynamic games, finding the OLGNE of the problem at hand can be linked to solving a single constrained optimal control problem, the solutions of which correspond to the OLGNE in the original dynamic game. Our key insight is that the solution to a multi-agent trajectory optimization problem can be seen as a constrained dynamic potential game, and its OLGNE can be found by solving a single constrained optimal control problem. Based on the type of potential function and their relation to agents' cost functions, there are various types of potential games in the literature, such as exact, weighted, and ordinal potential games \cite{monderer1996potential}.

Static versions of exact potential games were defined in \cite{monderer1996potential}. The following definition formally defines the dynamic version of exact constrained potential dynamic games.

\begin{definition}\label{def:exact-potential-condition}
A given game $\mathcal{G}:= \left( \mathcal{N}, \{\Gamma_i\}_{i\in\mathcal{N}}, \{J_i\}_{i\in\mathcal{N}},\{\mathcal{C}_k\}_{k\in\{0,\ldots, T\}},f,x_0 \right)$ is an exact constrained potential dynamic game if there exists a potential functional of the form $P(x,u) = L_T(x_T) + \sum_{k=0}^{T-1}L_k(x_k,u_k)$ such that for every agent $i\in \mathcal{N}$ and for every pair of feasible open-loop strategies $\gamma^i,\nu^i \in \Gamma^i$, we have the following relation when the strategy of all the other agents are fixed to be feasible strategies $\gamma^{-i} \in \Gamma^{-i}$:
\begin{equation}\label{eq:potential-condition-exact}
    J^i(x,u) - J^i(x',u') = P(x,u) - P(x',u'),
\end{equation}
where $(x,u)$ is the feasible system trajectory generated by strategies $\{\gamma^i,\gamma^{-i}\}$, and $(x',u')$ denotes the feasible system trajectory generated by $\{\nu^i,\gamma^{-i}\}$.
\end{definition}

Intuitively, condition~\eqref{eq:potential-condition-exact} implies that a game is a potential game if the change in the costs of each agent when they change their policy can be captured by a joint potential function when the strategies of all the other agents are fixed. 

In our previous works \cite{kavuncu2021potential, bhatt2022efficient}, we discussed how exact potential games can be employed for efficient and fast multi-agent motion planning. However, due to the strict conditions in Definition \ref{def:exact-potential-condition}, the type of cost functions considered for agents are very limited in the case of exact potential dynamic games. For example, it can only capture costs where the cost terms that couple agents' decisions among one another are symmetric and the same among the agents. We will relax the strict assumptions in our previous work by considering a more general class of potential dynamic games.

We will consider weighted constrained potential dynamic games (WCPDGs), which can capture various realistic cost functions for multi-agent planning. Let $w = (w^i_k)_{i\in\mathcal{N}, \; 0\leq k \leq T}$ be a set of positive numbers which we will call \emph{weights}. Taking motivation from \cite{monderer1996potential}, we define WCPDGs in the present context as follows:

\begin{definition}\label{def:potential-condition}
A given game $\mathcal{G}:= \left( \mathcal{N}, \{\Gamma_i\}_{i\in\mathcal{N}}, \{J_i\}_{i\in\mathcal{N}},\{\mathcal{C}_k\}_{k\in\{0,\ldots, T\}},f,x_0 \right)$ is a WCPDG with positive weights $w = (w^i_k)_{i\in\mathcal{N}, \; 0\leq k \leq T}$ if there exists a potential functional of the form $P(x,u) = L_T(x_T) + \sum_{k=0}^{T-1}L_k(x_k,u_k)$ such that for every agent $i\in \mathcal{N}$ and for every pair of feasible open loop strategies $\gamma^i,\nu^i \in \Gamma^i$, we have the following when the feasible strategy of all the other agents are fixed to be $\gamma^{-i} \in \Gamma^{-i}$: 
\begin{align}\label{eq:potential-condition}
    & J^i(x,u) - J^i(x',u') = w^i_T(L_T(x_T) - L_T(x^\prime_T)) \nonumber \\
    & \qquad \qquad + \sum_{k=0}^{T-1}w^i_k(L_k(x_k,u_k) - L_k(x^\prime_k,u^\prime_k)),
\end{align}
where $(x,u)$ is the feasible system trajectory generated by strategies $\{\gamma^i,\gamma^{-i}\}$, and $(x',u')$ denotes the feasible system trajectory generated by $\{\nu^i,\gamma^{-i}\}$.
\end{definition}

In a nutshell, the condition in \eqref{eq:potential-condition} suggests that the game is a WCPDG if the change in the combined running and terminal costs for each agent can be reflected in the changes to the potential function $P$ with its running and terminal costs multiplied by fixed positive weights. Essentially, there exists a potential function $P$ that can gauge the rate of change in an agent's cost when the strategies of the other agents remain unchanged.

We consider the following assumptions for the rest of the article:

\begin{assumption}\label{assum-1}
Agents' running costs $L^i_k$ and terminal costs $L^i_T$ are continuously differentiable on $\mathcal{X}\times\mathcal{U}$ and $\mathcal{X}$, respectively.
\end{assumption}

\begin{assumption}\label{assum-2}
The systems dynamics $f$ and constraints $g$ are continuously differentiable in $\mathcal{X}\times\mathcal{U}$ and satisfy some regularity conditions (such as the linear independence of gradients or the Mangasarian-Fromovitz constraint qualification).
\end{assumption}

In the following, we discuss essential circumstances under which a game will be a WCPDG. This provides us with mathematical tools to examine if a game is a potential game for a given set of agents' individual cost functions. The following lemma provides the necessary and sufficient conditions under which a dynamic game is a WCPDG.

\begin{lemma}\label{lemma-1}
A noncooperative dynamic game $\mathcal{G} := \left( \mathcal{N}, \{\Gamma_i\}_{i\in\mathcal{N}}, \{J_i\}_{i\in\mathcal{N}},\{\mathcal{C}_k\}_{k\in\{0,\ldots,T\}},f,x_0 \right)$ is a WCPDG with positive weights $w = (w^i_k)_{i\in\mathcal{N}, \; 0\leq k \leq T}$ if and only if there exists a potential function $P(x,u) = L_T(x_T) + \sum_{k=0}^{T-1}L_k(x_k,u_k)$ such that for every agent $i \in \mathcal{N}$, every time step $ 
0 \leq k \leq T-1$, and every feasible trajectory $\{x,u\}$ we have:
\begin{align}\label{eq1:lemma1}
\frac{\partial L^i_k(x_k,u_k)}{\partial x^i_k} &= w^i_k\left(\frac{\partial L_k(x_k,u_k)}{\partial x^i_k}\right), \nonumber \\
\frac{\partial L^i_k(x_k,u_k)}{\partial u_k^i} &= w^i_k\left(\frac{\partial L_k(x_k,u_k)}{\partial u_k^i}\right), \end{align}
and for the final time step, we have:
\begin{align}\label{eq2:lemma1}
\frac{\partial L^i_T(x_{T})}{\partial x^i_{T}} &= w^i_T\left(\frac{\partial L_T(x_{T})}{\partial x^i_{T}}\right), \; \forall i \in \mathcal{N}.
\end{align}
\end{lemma}

\begin{proof}
From \eqref{eq1:lemma1}, for each agent $i\in \mathcal{N}$ and for all $k=0,1,\ldots,T-1$, we have that
\begin{align}
    \frac{\partial}{\partial x^i_k}\left( L^i_k(x_k,u_k) - w^i_k(L_k(x_k,u_k)) \right) = 0, \nonumber\\
    \frac{\partial}{\partial u^i_k}\left( L^i_k(x_k,u_k) - w^i_k(L_k(x_k,u_k)) \right) = 0,
\end{align}
and from \eqref{eq2:lemma1}, for the final time step, we have:
\begin{align}
    \frac{\partial}{\partial x^i_{T}}\left( L^i_T(x_{T}) - w^i_T(L_T(x_{T})) \right) = 0. \nonumber
\end{align}
This means that the difference between each agent's costs and the game potential function should not depend on $x^i_k$ and $u^i_k$. Therefore, we can rewrite the above condition as:
\begin{align}\label{eq1:difference}
    L^i_k(x_k,u^i_k,u^{-i}_k)= w^i_kL_k(x_k,u^i_k,u^{-i}_k) + \Theta^i_k(x_k^{-i},u_k^{-i}),
\end{align}
for every $0\leq k \leq T-1$, and
\begin{align}\label{eq2:difference}
    L^i_T(x_{T}) = w^i_TL_T(x_{T}) + \Theta^i_T(x^{-i}_{T}),
\end{align}
for some functions $\{\Theta^i_k\}_{k\in\{0,1,\ldots,T\}}$. Let $u^i$ and ${u^\prime}^i$ be the collection of action vectors for agent $i$ corresponding to the different feasible strategies $\gamma^i$ and $\nu^i$. Furthermore, let the state trajectories of agent $i$ generated by those strategies be $x^i$ and ${x^\prime}^i$. It should be noted that since all agents' dynamics are separate and we are keeping the strategies of all the other agents except agent $i$ to be the same, $x^{-i}$ and $u^{-i}$ will remain the same. Therefore, \eqref{eq1:difference} and \eqref{eq2:difference} holds true for all $u^i_k\in\mathcal{U}^i$ and we have:
\begin{align}\label{eq1:difference-2}
    L^i_k(x^\prime_k,{u^\prime}^i_k,u^{-i}_k) = w^i_kL_k(x^\prime_k,{u^\prime}^i_k,u^{-i}_k) + \Theta^i_k(x_k^{-i},u_k^{-i}),
\end{align}
for every $0\leq k \leq T-1$, and
\begin{align}\label{eq2:difference-2}
    L^i_T(x^\prime_{T}) & = w^i_T(L_T(x^\prime_{T})) + \Theta^i_T({x^\prime}^{-i}_{T}).
\end{align}
By subtracting \eqref{eq1:difference-2} from \eqref{eq1:difference} and summing over all $k$ and adding the difference of \eqref{eq2:difference-2} from \eqref{eq2:difference}, we obtain \eqref{eq:potential-condition} which completes if direction of the proof. 

For the only if part, taking the partial derivative of \eqref{eq:potential-condition} with respect to $x^i_k$ and $u^i_k$ for $0\leq k \leq T-1$, we obtain \eqref{eq1:lemma1} for every agent $i \in \mathcal{N}$. Similarly, taking the partial derivative of \eqref{eq:potential-condition} with respect to $x^i_T$, we obtain \eqref{eq2:lemma1} for all $i \in \mathcal{N}$. This completes the proof of the other direction.
\end{proof}

Lemma \ref{lemma-1} essentially suggests that a game is a WCPDG if and only if, for every agent, the partial derivative of its cost function with respect to their states and actions are a scaled version of the partial derivatives of the potential function with respect to the agent's states and actions. We can rephrase these conditions through the following lemma, which suggests that a game is a WCPDG if every agent's cost function can be decomposed into two terms where one term represents the potential function, and the other term is independent of the states and actions of the agent itself. 

\begin{lemma}\label{lemma-2}
A game $\mathcal{G}$ is a WCPDG with positive weights $w = (w^i_k)_{i\in\mathcal{N}, \; 0\leq k \leq T}$ if and only if the running costs and terminal costs of every agent $i\in\mathcal{N}$ can be expressed as the sum of a term that is common to all agents plus another term that depends neither on its own action nor on its own state components, i.e., for every agent $i\in \mathcal{N}$ and for all feasible trajectories $\{x,u\}$, we have the following for some function $\Theta_k^i(\cdot)$ for $0\leq k \leq T-1$ and $\Theta_T^i(\cdot)$ for $k=T$:
\vspace{-0.2cm}
\begin{align}\label{eq:lemma2}
    L^i_k(x_k,u^i_k,u^{-i}_k) & = w^i_k(L_k(x_k,u^i_k,u^{-i}_k)) + \Theta^i_k(x_k^{-i},u_k^{-i}), \nonumber \\
    L^i_T(x_{T}) & = w^i_T(L_T(x_{T})) + \Theta^i_T(x^{-i}_{T}).
\end{align}
\end{lemma}
\begin{proof}
By taking the partial derivatives of \eqref{eq:lemma2} with respect to $x^i_k$ and $u^i_k$ for $0\leq k \leq T-1$ and $x^i_T$, we obtain \eqref{eq1:lemma1} and $\eqref{eq2:lemma1}$. Therefore, we can apply Lemma~\ref{lemma-1} to prove that $\mathcal{G}$ is a WCPDG. This completes the if direction of the proof. Now, assume that $\mathcal{G}$ is a WCPDG. Therefore, conditions of Lemma \ref{lemma-1} hold true. As proven in Lemma \ref{lemma-1}, the conditions of Lemma \ref{lemma-1} leads to \eqref{eq1:difference-2} and \eqref{eq2:difference-2} that are same as \eqref{eq:lemma2}. This completes the only if direction of the proof, completing the proof.
\end{proof}
An important property of potential dynamic games is that a set of GNE can be found by solving a single constrained multivariable optimal control problem. We state this important property in the following result.

\begin{theorem}\label{thm-1}
Suppose that $\mathcal{G} := \left( \mathcal{N}, \{\Gamma_i\}_{i\in\mathcal{N}}, \{J_i\}_{i\in\mathcal{N}},\{\mathcal{C}_k\}_{k\in\{0,\ldots,T\}},f,x_0 \right)$ is a WCPDG, and in addition, Assumptions~\ref{assum-1} and~\ref{assum-2} hold. Then, a solution to the multivariable optimal control problem:
\begin{align}\label{mopc}
      \underset{\gamma \in \Gamma}{\text{minimize}} \quad & L_T(x_{T}) + \sum_{k=0}^{T-1}L_k(x_k,u_k) \nonumber\\
     \text{subject to} \quad& x_{k+1} = f_k(x_k,u_k) \nonumber \\
    & g_k(x_k,u_k) \leq 0,
\end{align}
is an OLGNE for $\mathcal{G}$.
\end{theorem}
\begin{proof}
The proof is provided in Appendix \ref{proof-thm-1}.
\end{proof}

In contrast to multiple constrained coupled optimal control problems, Theorem \ref{thm-1} essentially provides us with a single constrained optimal control problem that we can solve for finding the OLGNE of a WCPDG. The obtained single-constrained optimal control problem can be solved efficiently using any off-the-shelf optimizer. This significantly reduces the computational complexity of finding OLGNE. It should be noted that one does not need a centralized coordinator to solve \eqref{mopc} and provide OLGNE trajectories to all the agents. Each agent keeps a copy of \eqref{mopc} and solves it independently to obtain the OLGNE and extract their individual controls out of the solution. We acknowledge that in practice, there could be scenarios where multiple OLGNEs could exist, and some kind of strategy alignment mechanism might be required for such scenarios. This has been studied in recent work \cite{peters2020inference}. We plan to extend our work to account for multiple OLGNEs however, it is beyond the scope of the current work.

In the following sections, we characterize conditions under which a constrained dynamic game is a WCPDG and show that a large number of realistic multi-agent interactions can be modeled as potential dynamic games.

\section{Dyadic Interactions}\label{sec:dyadic-interactions}
In this section, we consider dyadic interactions, which involve interactions between two agents in a non-cooperative setting. We will show that under certain conditions, the underlying dynamic game of two-agent interactions is always a WCPDG. Therefore, optimizing the potential function can efficiently solve the resulting game. We formalize these conditions in this section.

In this case, we denote the agents by indices 1 and 2. For the two agent case, the state and control input of the system at time $k$ will be $x_k = (x^{1^\top}_k,x^{2^\top}_k)^\top$ and $u_k = (u^{1^\top}_k, u^{2^\top}_k)^\top$. Let the stage-wise cost of both agents be given by
\begin{subequations}\label{eq:dya_cost_1}
\begin{align}\label{eq:dya_cost_1a}
    L^1_k(x_k,u_k) & = L^{11}_k(x^1_k, u^1_k) + c^1_kL^{12}_k(x^1_k,x^2_k) \\
    L^2_k(x_k, u_k) & = L^{22}_k(x^2_k, u^2_k) + c^2_kL^{21}_k(x^2_k,x^1_k), \label{eq:dya_cost_1b} 
\end{align}
\end{subequations}
for $0\leq k \leq T-1$, and 
\begin{subequations}\label{eq:dya_cost_2}
    \begin{align}
    L^1_T(x_T) & = L^{11}_T(x^1_T) + c^1_TL^{12}_T(x^1_T,x^2_T) \label{eq:dya_cost_2a} \\
    L^2_T(x_T) & = L^{22}_T(x^2_T) + c^2_TL^{21}_T(x^2_T,x^1_T). \label{eq:dya_cost_2b},
\end{align}
\end{subequations}
for the final time step $k=T$.
Note that $c^1_k>0, c^2_k > 0, k \in \{0,1,\ldots, T\}$ are hyper parameters in the agents' cost functions. Intuitively, for each agent, the stage-wise cost is a sum of agent-specific cost ($L^{ii}_k(\cdot)$) and the inter-agent cost ($L^{ij}_k(\cdot)$). We assume that inter-agent costs do not depend on agents' actions. Agent-specific cost terms can be considered as terms that capture costs such as goal reaching, accomplishing certain individual tasks, satisfying individual preferences, avoiding static obstacles, controlling effort penalties, etc. Inter-agent costs are cost terms that capture the coupling between the two agents, such as preferences for staying out of close proximity of the other agent. It should be noted that even though we have collision avoidance constraints encoded in $g(\cdot,\cdot)$, we observe that different agents have different preferences for how close they get to the other agents in the environment.
We consider the following assumption about the inter-agent costs that capture the coupling among the agents.
\begin{assumption}\label{assum-symm-cost} For every time step $ 1\leq k \leq T$, we assume that $L^{12}_k(x^1_k,x^2_k) = L^{21}_k(x^2_k,x^1_k)$ for every tuple of states $(x^1_k, x^2_k)$. 
\end{assumption}
Assumption~\ref{assum-symm-cost} states that the functional form of the terms of the inter-agent costs is the same for both agents. This is a valid assumption as we have coefficients $c^1_k$ and $c^2_k$ to capture the asymmetries between the agents on how they want to consider the inter-agent costs. For example, suppose the inter-agent cost term captures the proximity preference cost between the two agents. In that case, generally, such costs are captured by distance functions, which are the same for all the agents. However, if one agent cares more about avoiding close proximity than the other agent, that agent will have a higher weight on the proximity preference cost than the other.

We present the following lemma, which states that such dyadic interactions are always WCPDGs.

\begin{lemma}\label{lemma-3}
A dyadic interaction with stage wise costs~\eqref{eq:dya_cost_1} and \eqref{eq:dya_cost_2} is always a WCPDG with weights $w^1_k = \frac{1}{c_k^2}$ and $w^2_k = \frac{1}{c_k^1}$ if Assumption \ref{assum-symm-cost} holds. Furthermore, the corresponding potential function is $P(x,u) =L_T(x_T) + \sum_{k=0}^{T-1}L_k(x_k,u_k)$ where
\begin{align}
    L_k(x_k,u_k) & = c_k^2L^{11}_k(x^1_k,u^1_k) + c_k^1L^{22}_k(x^2_k,u^2_k) \nonumber \\
    & \quad + c_k^1c_k^2L^{12}_k(x^1_k,x^2_k),
\end{align}
and
\begin{align}
    L_T(x_T) = c_T^2L^{11}_T(x^1_T) + c_T^1L^{22}_T(x^2_T) + c_T^1c_T^2L^{12}_T(x^1_T,x^2_T).
\end{align}

\begin{proof}
From \eqref{eq:dya_cost_1a} we have that,
\begin{align}
    L^1_k(x_k,u_k) & = L^{11}_k(x^1_k, u^1_k) + c_k^1L^{12}_k(x^1_k,x^2_k) \nonumber \\
    & = \frac{1}{c_k^2}\left( c_k^2L^{11}_k(x^1_k, u^1_k) + c_k^1c_k^2L^{12}_k(x^1_k,x^2_k) \right) \nonumber \\
    & = \frac{1}{c_k^2}L_k(x_k,u_k) - \frac{c_k^1}{c_k^2}L_k^{22}(x^2_k,u^2_k) \nonumber \\
    & = w^1_kL_k(x_k,u_k) + \Theta_k^1(x^2_k,u^2_k),
\end{align}
where $w^1_k = \frac{1}{c_k^2}$ and $\Theta_k^1(x^2_k,u^2_k) = - \frac{c_k^1}{c_k^2}L_k^{22}(x^2_k,u^2_k)$. Similarly, from Assumption \ref{assum-symm-cost} we can write,
\begin{align}
    L^2(x_k,u_k) &= w^2_kL_k(x_k,u_k) + \Theta_k^2(x^1_k,u^1_k),
\end{align}
where, $w^2_k= \frac{1}{c_k^2}$, $\Theta_k^2(x^1_k,u^1_k) = - \frac{c_k^2}{c_k^1}L_k^{11}(x^1_k,u^1_k)$. 
Likewise, for the final time step, T, we have
\begin{align}
    L^1_T(x_T) &= w^1_TL_T(x_T) + \Theta^1_T(x^2_T) \\
    L^2_T(x_T) &= w^2_TL_T(x_T) + \Theta^2_T(x^1_T),
\end{align}
where $ \Theta^1_T(x^2_T) = -\frac{c_T^1}{c_T^2}L^{22}_T(x^2_T)$, and $\Theta^2_T(x^1_T) = -\frac{c_T^2}{c_T^1}L^{11}_T(x^1_T)$. Therefore, applying Lemma~\ref{lemma-2}, we have that such a dyadic interaction will be a WCPDG.
\end{proof}

\end{lemma}

As mentioned previously, this result extends the symmetry assumptions required for the inter-agent costs from our previous work to arbitrary values of cost coefficients in Lemma \ref{lemma-3} for dyadic interactions. This captures many real-world robotics interactions where there is an asymmetry in the way agents assign costs to each other. The following section discusses a more general version of the result in Lemma \ref{lemma-3} for multi-agent interactions.

\section{Multi-agent Interactions}\label{multi-agent-inter}\label{sec:multi-agent-interactions}
In this section, we generalize our results from the previous section and go beyond two agents. We characterize conditions under which the interaction of more than two agents is a WCPDG.

Consider the interactions of $N$ agents in an environment. For every agent $i \in \mathcal{N} $, let the stage-wise and terminal costs be of the form,
\begin{subequations}
\begin{align}\label{eq:multi_1_eq_1a}
    L^i_k(x_k,u_k) & = L^{ii}_k(x^i_k, u^i_k) + \sum_{\substack{j=1 \\
    j\neq i}}^{N}c_k^{ij}L^{ij}_k(x^i_k,x^j_k),
\end{align}
and
\begin{align}\label{eq:multi_1_eq_1b}
    L^i_T(x_T) & = L^{ii}_T(x^i_T) + \sum_{\substack{j=1 \\
    j\neq i}}^{N}c_k^{ij}L^{ij}_T(x^i_T,x^j_T).
\end{align}
\end{subequations} We also assume that $c_k^{ij} > 0$ for all $i,j \in \mathcal{N}$ and for all $k \in \{0,1,\ldots,T\}$. Intuitively, the stage-wise cost for each agent is the sum of a term that depends on the agent's own state and actions $L^{ii}_k(x^i_k, u^i_k)$ and the sum of terms that capture the pairwise interaction of the agent $i$ with other agents $j$, $L^{ij}_k(x^i_k,x^j_k)$. As mentioned earlier, the first term can be thought of as a term that captures the individual preferences of the agent while the coupling between the agents is captured through the other terms. In the following lemmas, we show that under certain conditions, such multi-agent interactions become WCPDGs. Same as before, we make an additional assumption that inter-agent cost terms have the same functional form among the agents.
\begin{assumption}\label{assum-sym-cost-2}
    Assume that for any two agents $i,j \in \mathcal{N}, i\neq j$, and every time step $1 \leq k\leq T-1$, we have that $L^{ij}_k(x^i_k,x^j_k) = L^{ji}_k(x^j_k,x^i_k)$ for any $x_k^i \in \mathcal{X}^i$ and $x_k^j \in \mathcal{X}^j$., Likewise, for the final time step, we have $L^{ij}_T(x^i_T,x^j_T) = L^{ji}_T(x^j_T,x^i_T)$ for any set of final states $x^i_T \in \mathcal{X}^i$ and $x^j_T  \in \mathcal{X}^j$.
\end{assumption}
Similar to the previous section, the above assumption requires that the functional form of the couplings is the same among the agents, and we capture the asymmetries between the agents through the inter-agent cost coefficients ($c^{ij}_k$).

In order to better explain our ideas, we consider a running example of three agents interacting with each other in an environment. Fig. \ref{fig:schematic-1} provides a schematic figure to represent this interaction. Each agent has two types of costs: agent-specific and inter-agent costs. The inter-agent costs are shown on directed arrows, which show the cost a particular agent incurs from the other agent. For example, the arrow from agent 1 to agent 2 captures the cost $c^{12}_kL^{12}_k(x^1_k,x^2_k)$ that is the cost incurred to agent 1 from agent 2.
\begin{figure}
    \centering
    \includegraphics[scale=0.18]{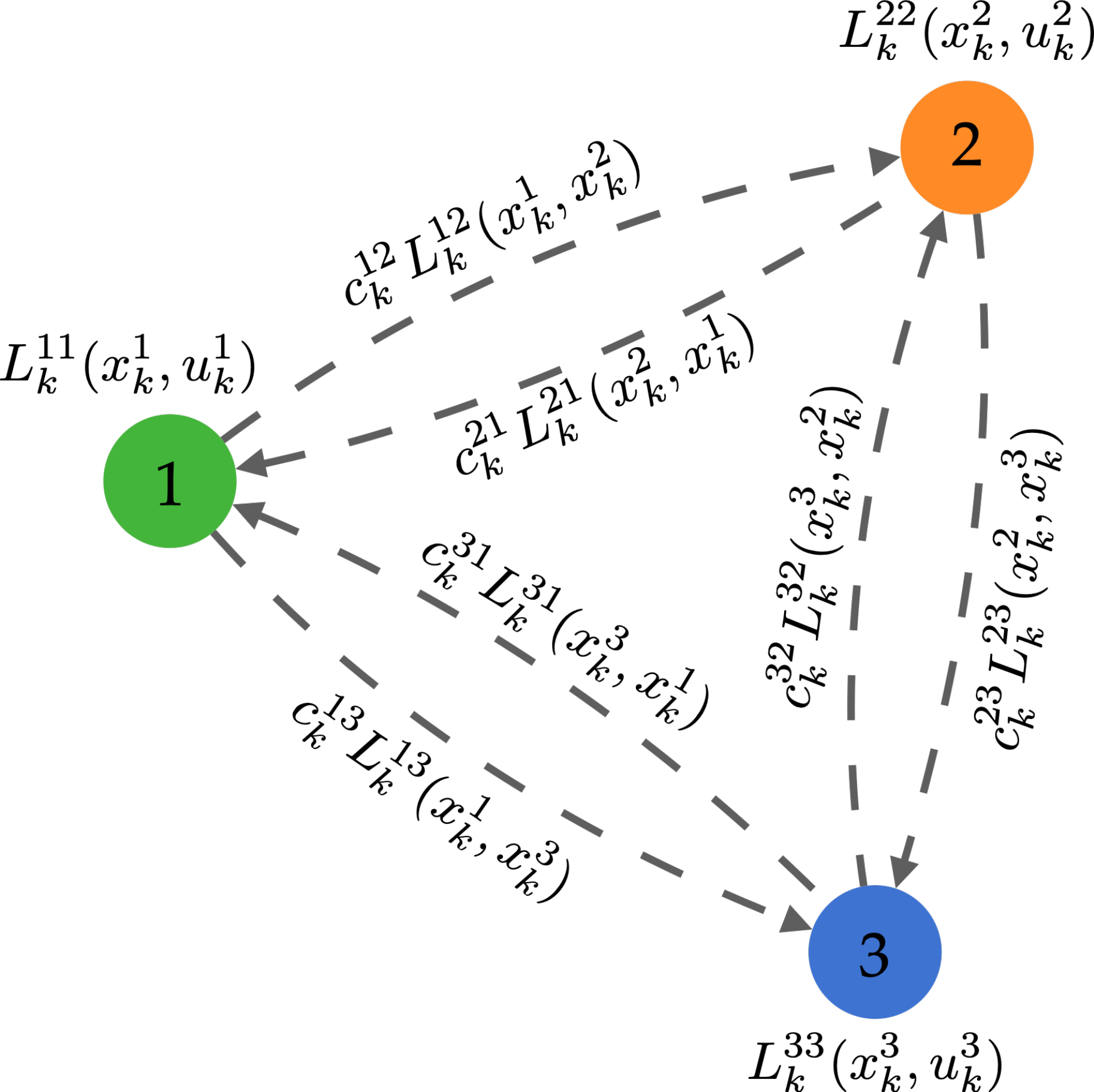}
    \caption{Schematic of the cost function dependencies for a running example with three agents interacting with one another.}
    \label{fig:schematic-1}
\end{figure}



Consider a scenario of multi-agent interactions where each agent incurred a cost for getting in close proximity with all the other agents in the environment with the same coefficient, i.e., $c^{ij} = c^{i}$. This means that each agent has the same inter-agent cost coefficient for all the other agents in the environment, i.e., \emph{each agent treats all the other agents similarly}. The schematic of such an interaction is shown in Fig.~\ref{fig:schematic-2}.
In the following lemma, we show that such an interaction becomes an instance of a $w$-potential game and, therefore, we can associate a potential function with it, minimizing which results in fast and efficient trajectory planning.
\begin{lemma}\label{lemma-4} Under Assumption \ref{assum-sym-cost-2}, a multi-agent interaction game $\mathcal{G} := \left( \mathcal{N}, \{\Gamma_i\}_{i\in\mathcal{N}}, \{J_i\}_{i\in\mathcal{N}},\{\mathcal{C}_k\}_{k\in\{0,\ldots,T\}},f,x_0 \right)$ with stage-wise and terminal costs~\eqref{eq:multi_1_eq_1a} and \eqref{eq:multi_1_eq_1b} is a $w$-potential game with weights 
\begin{equation}
    w^i_k= \frac{1}{\Pi_{\substack{j=1 \\ j\neq i}}^N c_k^j },
\end{equation} 
if for every agent $i$ in $\mathcal{N}$, we have $c_k^{ij} = c_k^i$ for all $j \in \mathcal{N}\setminus\{i\}$. The corresponding potential function of the game is $P(x,u) =L_T(x_T) + \sum_{k=0}^{T-1}L_k(x_k,u_k)$ where
\begin{align}
    L_k(x_k,u_k) & = \sum_{i=1}^N (\Pi_{\substack{j=1 \\ j\neq i}}^N c_k^j ) L^{ii}_k(x^i_k, u^i_k) \nonumber \\
    & \quad + (\Pi_{\substack{l=1}}^N c_k^l) \sum_{i=1}^{N-1} \sum_{j=i+1}^{N} L^{ij}_k(x^i_k,x^j_k), 
\end{align}
and
\begin{align}
    L_T(x_T) & = \sum_{i=1}^N (\Pi_{\substack{j=1 \\ j\neq i}}^N c_k^j ) L^{ii}_T(x^i_T) \nonumber \\
    & \quad + (\Pi_{\substack{l=1}}^N c_k^l) \sum_{i=1}^{N-1} \sum_{j=i+1}^{N} L^{ij}_T(x^i_T,x^j_T). 
\end{align}
\end{lemma}

\begin{figure}
    \centering
    \includegraphics[scale=0.18]{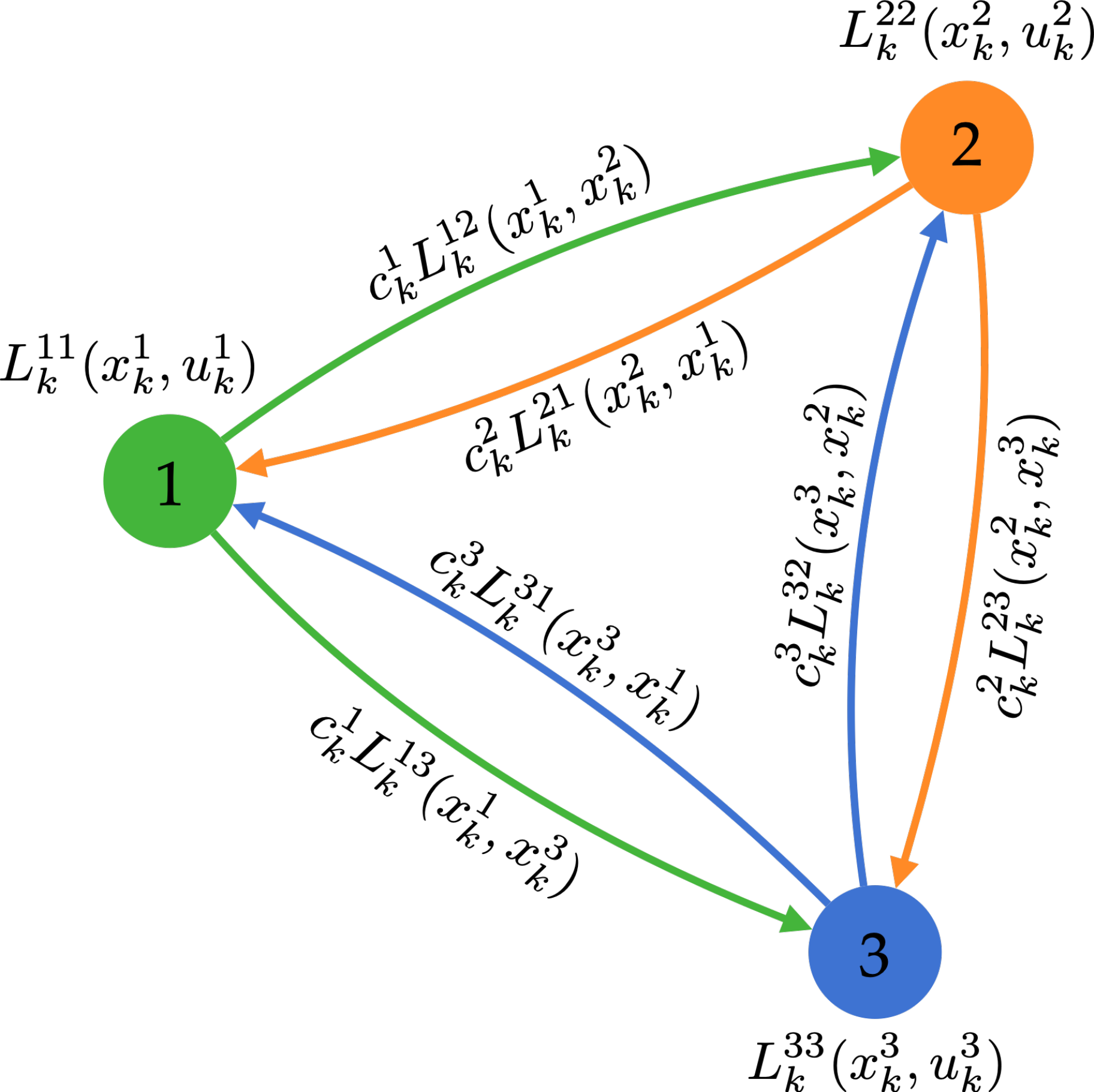}
    \caption{Schematic of an interaction where each agent in the environment treats all the other agents in a similar fashion. In such scenarios, the underlying dynamic game becomes a weighted potential dynamic game. Note that in this figure, the edges with the same color indicate the same value of the cost coefficient.}
    \label{fig:schematic-2}
\end{figure}

\begin{proof}
From \eqref{eq:multi_1_eq_1a} and Assumption \ref{assum-sym-cost-2}, we have that
\begin{align}
    L^i_k(x_k,u_k) & = L^{ii}_k(x^i_k, u^i_k) + c_k^i\sum_{\substack{j=1 \\
    j\neq i}}^{N}L^{ij}_k(x^i_k,x^j_k) \nonumber \\
    & = \frac{1}{\Pi_{\substack{j=1 \\ j\neq i}}^N c_k^j } \Bigg[ (\Pi_{\substack{j=1 \\ j\neq i}}^N c_k^j )L^{ii}_k(x^i_k, u^i_k) \Bigg . \nonumber \\ & \qquad \qquad + \Bigg. (\Pi_{\substack{l=1}}^N c_k^l) \sum_{\substack{j=1 \\
    j\neq i}}^{N}L^{ij}_k(x^i_k,x^j_k) \Bigg]. 
\end{align}
Therefore,
\begin{align}
    L^i_k(x_k,u_k) & = w^iL_k(x_k,u_k) - w^i\sum_{\substack{j=1 \\ j\neq i}}^N(\Pi_{\substack{l=1 \\ l\neq j}}^N c_k^l) L^{jj}_k(x^j_k, u^j_k) \nonumber \\
    & \quad - w^i(\Pi_{\substack{l=1}}^N c_k^l) \sum_{\substack{j=1 \\ j\neq i}}^{N-1} \sum_{\substack{l=j+1 \\ l\neq i}}^{N} L^{jl}_k(x^j_k,x^l_k) \nonumber \\
    & = w^iL_k(x_k,u_k) + \Theta^i_k(x_k^{-i},u_k^{-i}),
\end{align}
where $\Theta^i_k(x_k^{-i},u_k^{-i}) = - w^i\sum_{\substack{j=1 \\ j\neq i}}^N(\Pi_{\substack{l=1 \\ l\neq j}}^N c_k^l) L^{jj}_k(x^j_k, u^j_k) - w^i(\Pi_{\substack{l=1}}^N c_k^l) \sum_{\substack{j=1 \\ j\neq i}}^{N-1} \sum_{\substack{l=j+1 \\ l\neq i}}^{N} L^{jl}_T(x^j_T,x^l_T)$. Similarly, we can write
\begin{align}
    L^i_T(x_T) = w^i(L_T(x_{T})) + \Theta^i_T(x^{-i}_{T}),
\end{align}
where $ \Theta^i_T(x^{-i}_{T}) = - w^i\sum_{\substack{j=1 \\ j\neq i}}^N(\Pi_{\substack{l=1 \\ l\neq j}}^N c_k^l) L^{jj}_T(x^j_T) - w^i(\Pi_{\substack{l=1}}^N c_k^l) \sum_{\substack{j=1 \\ j\neq i}}^{N-1} \sum_{\substack{l=j+1 \\ l\neq i}}^{N} L^{jl}_T(x^j_T,x^l_T)$. Therefore, by applying Lemma~\ref{lemma-2}, we have that such a multi-agent interaction is a WCPDG.
\end{proof}

In the following, we consider another practical case that is of relevance to many real-world applications. Consider a scenario of multi-agent interactions where each agent is penalized similarly by all the other agents in the environment. This means that each agent is treated with the same inter-agent cost coefficient by all the other agents in the environment, i.e., $c^{ij} = c^{j}$ for all $i \in \mathcal{N}\setminus\{j\}$. The schematic of this type of interaction is shown in Fig.~\ref{fig:schematic-3}.
\begin{figure}
    \centering
    \includegraphics[scale=0.18]{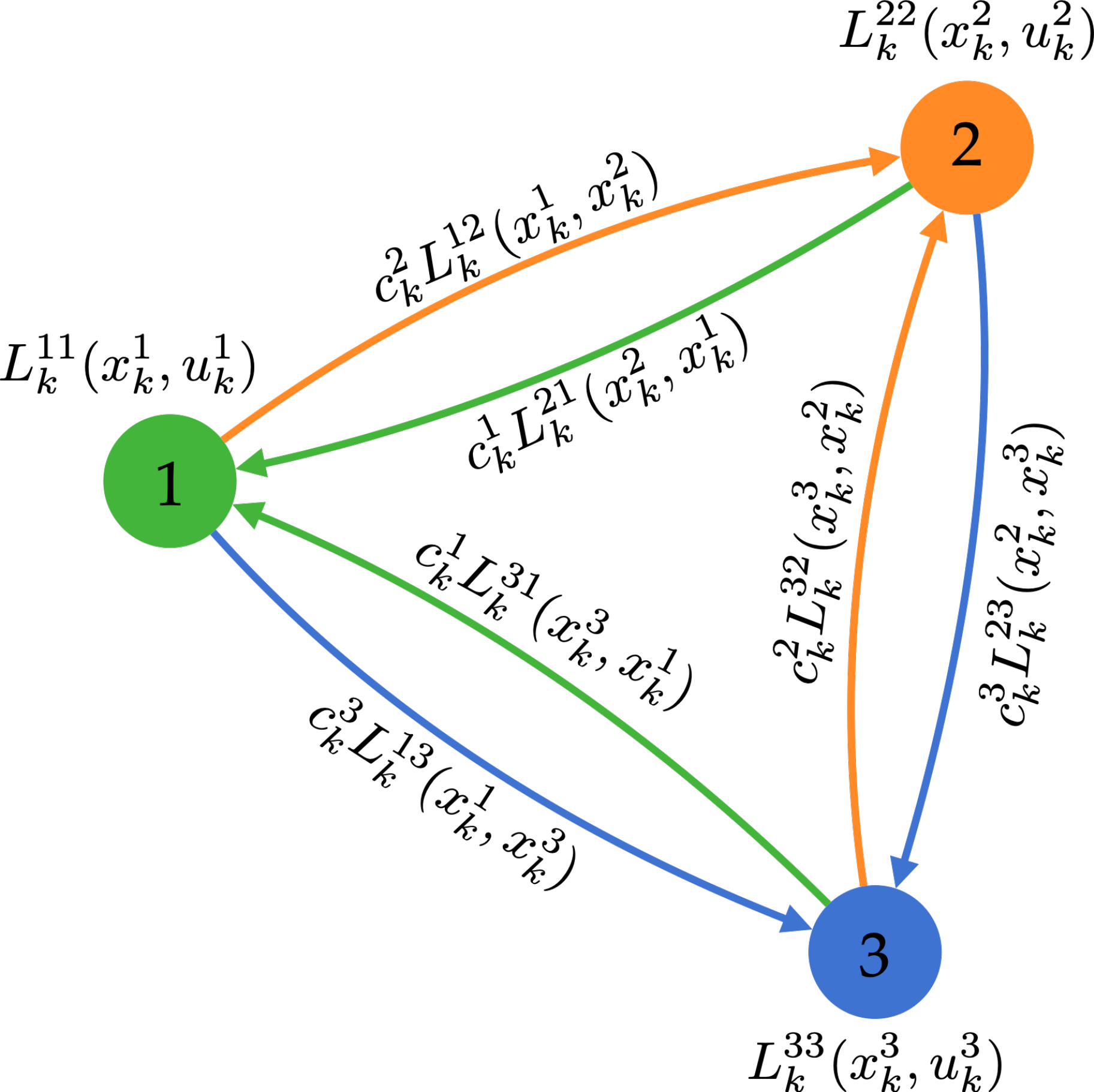}
    \caption{Schematic of interactions where each agent in the environment is treated by all the other agents in a similar fashion. In such scenarios, the underlying dynamic game becomes a weighted potential dynamic game. It should be noted that edges with the same color indicate the same value of the cost coefficient.}
    \label{fig:schematic-3}
\end{figure}

In the following lemma, we show that such an interaction also becomes an instance of a $w$-potential game; therefore, we can associate a potential function with it.




\begin{lemma}\label{lemma-5} Under Assumption \ref{assum-sym-cost-2},
a multi-agent interaction game $\mathcal{G} := \left( \mathcal{N}, \{\Gamma_i\}_{i\in\mathcal{N}}, \{J_i\}_{i\in\mathcal{N}},\{\mathcal{C}_k\}_{k\in\{0,\ldots,T\}},f,x_0 \right)$ with stage-wise and terminal costs according to \eqref{eq:multi_1_eq_1a} and \eqref{eq:multi_1_eq_1b} is a $w$-potential game with weights $w^i= \frac{1}{c_k^i}$ if $c_k^{ij} = c_k^j$ for every $i \in \mathcal{N}\setminus\{j\}$. The corresponding potential function is $P(x,u) =L_T(x_T) + \sum_{k=0}^{T-1}L_k(x_k,u_k)$ where
\begin{align}
    L_k(x_k,u_k) & = \sum_{i=1}^N c_k^iL^{ii}_k(x^i_k, u^i_k) + \sum_{i=1}^{N-1} \sum_{j=i+1}^{N} c_k^ic_k^j L^{ij}_k(x^i_k,x^j_k),
\end{align}
and
\begin{align}
    L_T(x_T) & = \sum_{i=1}^N c_k^i L^{ii}_T(x^i_T) + \sum_{i=1}^{N-1} \sum_{j=i+1}^{N} c_k^i c_k^j L^{ij}_T(x^i_T,x^j_T) .
\end{align}
\end{lemma}

\begin{proof}
From \eqref{eq:multi_1_eq_1a} and Assumption \ref{assum-sym-cost-2}, we have that
\begin{align}
    L^i_k(x_k,u_k) & = L^{ii}_k(x^i_k, u^i_k) + \sum_{\substack{j=1 \\
    j\neq i}}^{N} c_k^j L^{ij}_k(x^i_k,x^j_k) \nonumber \\
    & = \frac{1}{c_k^i}\Bigg\{ c_k^i L^{ii}_k(x^i_k, u^i_k) + \sum_{\substack{j=1 \\
    j\neq i}}^{N} c_k^ic_k^j L^{ij}_k(x^i_k,x^j_k)\Bigg\}. 
\end{align}
Therefore,
\begin{align}
    L^i_k(x_k,u_k) & = w^iL_k(x_k,u_k) - \frac{1}{c_k^i}\Bigg\{ \sum_{\substack{j=1 \\ j\neq i}}^N L^{jj}_k(x_k,u_k) \Bigg. \nonumber \\
    & \quad \quad \Bigg. + \sum_{\substack{j=1 \\ j\neq i}}^{N-1} \sum_{\substack{l=j+1 \\ l\neq i}}^{N} L^{jl}_k(x^j_k,x^l_k) \Bigg \} \nonumber \\
    & = w^iL_k(x_k,u_k) + \Theta^i_k(x_k^{-i},u_k^{-i}),
\end{align}
where $\Theta^i_k(x_k^{-i},u_k^{-i}) = -\frac{1}{c_k^i}\{ \sum_{\substack{j=1 \\ j\neq i}}^N L^{jj}_k(x_k,u_k) + \sum_{\substack{j=1 \\ j\neq i}}^{N-1} \sum_{\substack{l=j+1 \\ l\neq i}}^{N} L^{jl}_k(x^j_k,x^l_k) \}$. Similarly, we can write
\begin{align}
    L^i_T(x_T) = w^i(L_T(x_{T})) + \Theta^i_T(x^{-i}_{T}),
\end{align}
where $ \Theta^i_T(x_T^{-i}) = \frac{1}{c_k^i}\{ \sum_{\substack{j=1 \\ j\neq i}}^N L^{jj}_T(x_T) + \sum_{\substack{j=1 \\ j\neq i}}^{N-1} \sum_{\substack{l=j+1 \\ l\neq i}}^{N} L^{jl}_T(x^j_T,x^l_T) \}$. Therefore, applying Lemma-\ref{lemma-2}, we have that such a multi-agent interaction will be a WCPDG.
\end{proof}

\begin{remark}
    Note that in this formulation, all agents are not necessarily required to interact with all other agents in the environment. For example, if agent $i$ and agent $j$ do not interact with each other, then the corresponding inter-agent cost term $L^{ij}_k(x^i_k,x^j_k)$ will be zero, and; therefore, that term won't appear in the potential function.
\end{remark}

Lemmas \ref{lemma-3}, \ref{lemma-4}, and \ref{lemma-5} provide realistic motion planning scenarios under which a game $\mathcal{G} := \left( \mathcal{N}, \{\Gamma_i\}_{i\in\mathcal{N}}, \{J_i\}_{i\in\mathcal{N}},\{\mathcal{C}_k\}_{k\in\{0,\ldots,T\}},f,x_0 \right)$ is a WCPDG. Therefore, we can use the result from Theorem \ref{thm-1} and associate a single constrained optimal control problem as described in \eqref{mopc} for finding equilibria of the game. If there are no constraints in the game, then solving \eqref{mopc} is an easy task as it can be solved using any standard optimization techniques of single agent optimal control. In particular, in unconstrained settings, we employ the Iterative Linear Quadratic Regulator (iLQR) \cite{li2004iterative} to solve the single agent optimal control problem. As the name suggests, the iterative linear quadratic regulator (iLQR) is an algorithm that iteratively refines both the control inputs and the value function estimates to improve the control policy. 

Furthermore, in constrained settings, to solve~\eqref{mopc}, we use the Augmented Lagrangian Trajectory Optimization (ALTRO) algorithm~\cite{howell2019altro}. ALTRO combines iterative-LQR (iLQR)~\cite{li2004iterative} with an augmented Lagrangian method to handle general state and input constraints and an active-set projection method for final ``solution polishing'' to achieve fast and robust convergence.

\section{Simulations Studies}\label{sec:simulations}
In this section, we provide the simulation results for our algorithm. First, we provide an analysis of two-player LQ games to show that our algorithm indeed recovers the open-loop Nash equilibrium of the original game. Then, we provide analysis for general nonlinear games to showcase the capabilities of our algorithm in realistic scenarios. Furthermore, we show through Monte Carlo analysis that our algorithm is faster than state-of-the-art methods.

\subsection{Dyadic Linear Quadratic Interactions}\label{lq-sim}

The main result of our work is to prove that we can compute the open-loop Nash equilibrium of a game by minimizing a single potential function if a game is WCPDG. In order to numerically verify this claim, we consider a two-agent Linear Quadratic Game (LQ Game). We choose this because, for linear quadratic games, exact closed-form solutions for computing open-loop Nash equilibria are available. LQ games are a special version of affine-quadratic games (Definition 6.1 in \cite{bacsar1998dynamic}) where the dynamics are entirely linear with no affine term:
\begin{definition}
A game is a linear quadratic (LQ Game) if the system dynamics are defined as
\begin{equation}
    x_{k+1} = f_k(x_k,u_k) = A_kx_k + B_ku_k,
\end{equation}
where $A_k$ is the state matrix and $B_k$ is the control matrix; and the stage-wise cost functions for every agent $i \in \mathcal{N}$ are \begin{align}
    L^i_k(x_k,u_k) = \begin{cases} \frac{1}{2}\left( \sum_{j\in \mathcal{N}} {u^j}^\top_k R^{ij}_k u^j_k \right), & k = 0 \\
    \frac{1}{2}\left( x_k^\top Q^i_k x_k + \sum_{j\in \mathcal{N}} {u^j}^\top_k R^{ij}_k u^j_k \right), & k \neq 0, T,
    \end{cases}
\end{align}
and
\begin{equation}
    L^i_T(x_T) = \frac{1}{2}x_T^\top Q^i_T x_T,
\end{equation}
where $Q^i$'s are symmetric positive semi-definite matrices, and $R^{ii}$'s are positive definite matrices.
\end{definition}

\begin{figure}
    \centering
    \includegraphics[width=\linewidth]{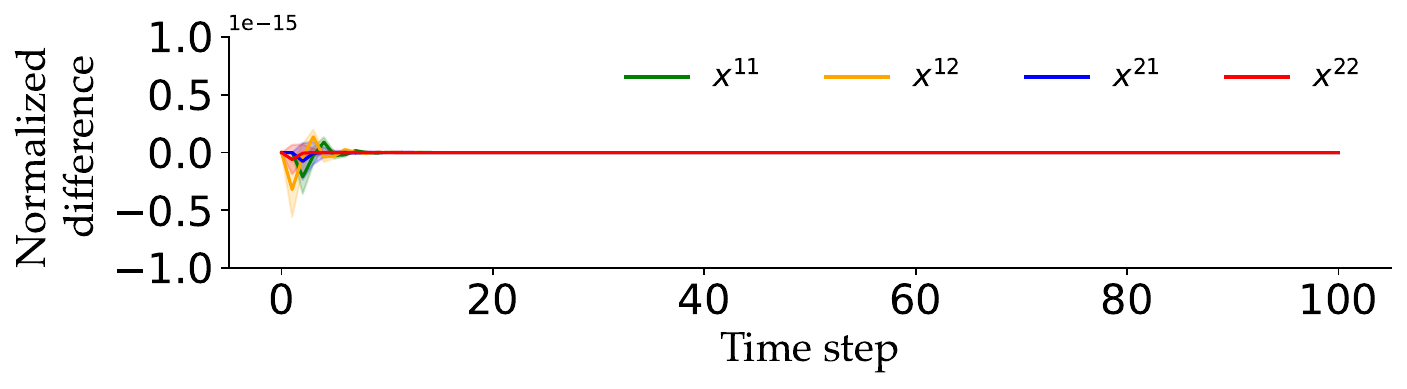}
    \caption{Comparison of trajectories for the LQ-game solved using potential minimization and exact open-loop Nash solution shows that the trajectories generated are close up to the system precision level.}
    \label{fig:pot_vs_ol}
\end{figure}

Now, we provide an example of a potential linear quadratic game and show how potential minimization recovers open-loop Nash equilibrium. Consider a two player LQ Game - $\mathcal{G}^{LQ}$ where $x^1 := (x^{{11}},x^{{12}})^\top$ and $x^2 := (x^{{21}},x^{{22}})^\top$ are state vector for both the agents with $x^{ij}$ being $j^{\text{th}}$ element of $i^{\text{th}}$ agent. We let the initial conditions of the game be $x^1_0 = (3,2)^\top$ and $x^2_0=(4,5)^\top$. Let the dynamics and cost matrices be
\begin{align}\label{eq:lqgame}
    & A_k = 
    \begin{bmatrix} 
        0 & 1 & 0 & 0 \\
        -1 & -1 & 0 & 0 \\
        0 & 0 & 0 & 1 \\
        0 & 0 & -1 & -1
    \end{bmatrix}, & B_k & = 
    \begin{bmatrix}
        0 & 0 \\
        1 & 0 \\
        0 & 0 \\
        0 & 1 \\
    \end{bmatrix} & \nonumber \\
    & Q^1_k =
    \begin{bmatrix}
        1 & -1 & 2 & 0 \\
        -1 & 5 & -1 & 1 \\
        2 & -1 & 6 & -2 \\
        0 & 1 & -2 & 4    
    \end{bmatrix}, & Q^2_k & =
    \begin{bmatrix}
        1 & -1 & 2 & 0 \\
        -1 & 4 & -1 & 1 \\
        2 & -1 & 6 & 0 \\
        0 & 1 & 0 & 2    
    \end{bmatrix} \nonumber \\
    & R^{11}_k = \begin{bmatrix}
        3
    \end{bmatrix}, R^{12}_k = 0, & 
    R^{21}_k & = 0, R^{12}_k = \begin{bmatrix}
        2
    \end{bmatrix}.
\end{align}

It is easy to verify using Lemma \ref{lemma-2} that $\mathcal{G}^{LQ}$ is a WCPDG with the weights $w_1 = w_2 = 1$ and  stage-wise potential functions 
\begin{align}
    L_k(x_k,u_k) = \begin{cases} \frac{1}{2}\left(  {u}^\top_k R_k u_k \right), & k = 0 \\
    \frac{1}{2}\left( x_k^\top Q_k x_k + {u}^\top_k R_k u_k \right), & k \neq 0, T,
    \end{cases}
\end{align}
and
\begin{equation}
    L_T(x_T) = \frac{1}{2}x_T^\top Q_T x_T,
\end{equation}
where, $R_k$ and $Q_k$ are
\begin{align}
     Q^1_k =
    \begin{bmatrix}
        1 & -1 & 2 & 0 \\
        -1 & 5 & -1 & 1 \\
        2 & -1 & 6 & 0 \\
        0 & 1 & 0 & 2    
    \end{bmatrix}, \quad & R_k = \begin{bmatrix}
        3 & 0 \\
        0 & 2
    \end{bmatrix}.
\end{align}


The open-loop Nash equilibria of the LQ game can be found using Theorem 6.2 from \cite{bacsar1998dynamic}. By formulating the game as a $w$-potential game, we need to just solve the resulting single-agent optimal control problem, which can be solved using Ricatti equations of standard discrete-time LQR. We consider 200 random initial conditions around $x^1_0$ and $x^2_0$. We plot the mean and standard deviation of the trajectories that result from potential function minimization and the exact open-loop trajectories of the system in Fig.~\ref{fig:pot_vs_ol}. As we can observe from Fig. \ref{fig:pot_vs_ol}, the trajectories obtained from Potential minimization and open-loop Nash equilibrium of the original game are close to the system precision level. This demonstrates that, indeed, we can recover the open-loop Nash equilibrium of the original game by potential function minimization for the case of $w$-potential game.

\begin{figure*}[th]
    \centering
    \includegraphics[scale = 0.67]{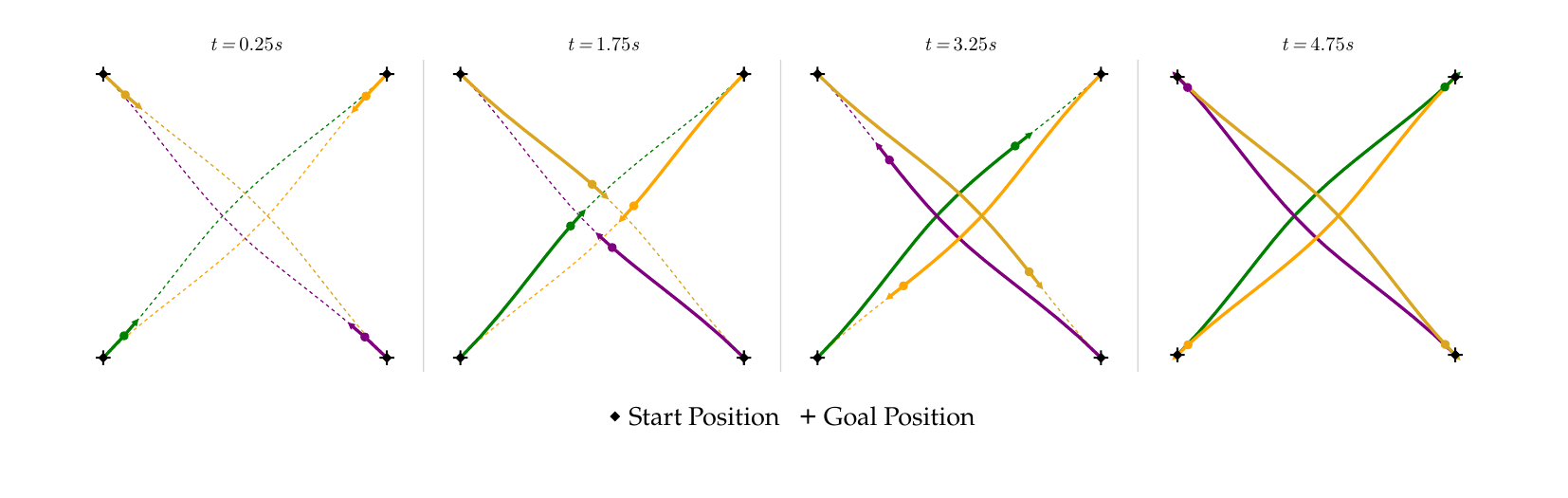}
    \vspace{-1.2cm}
    \caption{The snapshots of the trajectories found by our algorithm when four unicycle agents interact with each other and exchange their positions diagonally over a time interval of 5\emph{s}. Agents move from their start positions to their respective goal positions. In order to avoid collision with other agents, each agent deviates from its nominal trajectory and successfully reaches its goal position.}
    \label{fig:four_agents}
    \vspace{-0.7cm}
\end{figure*}

\subsection{Multi-agent Interactive Planning}\label{multi-sim}
In order to evaluate the performance of our method on systems with general nonlinear dynamics and costs, we consider an interaction among three agents in a planner navigation setting. Each agent wants to reach its goal position while avoiding collisions with other agents. We model each agent via discrete-time unicycle dynamics with a time step size of $\Delta t$:
\begin{align}\label{eq:unicycle_dynamics}
    & p^i_{k+1} = p^i_k + \Delta t \cdot v^i_k \cos{\theta^i_k}, \quad  q^i_{k+1} = q^i_k + \Delta t \cdot v^i_k \sin{\theta^i_k} \nonumber \\
    & \theta^i_{k+1} = \theta^i_k + \Delta t \cdot \omega^i_k, \quad v^i_{k+1} = v^i_{k} + \Delta t \cdot \alpha^i_k,
\end{align}
where $p^i_k$ and $q^i_k$ are $x$ and $y$ coordinates of the positions in the 2D plane, $\theta^i_k$ is the heading angle from positive x-axis, $v^i_k$ is the forward velocity, $\omega^i_k$ is the angular velocity, and $\alpha^i_k$ is the acceleration of agent $i$ at time step $k$. For each agent $i$, we let the state vector $x^i_k$ be $[p^i_k,q^i_k,\theta^i_k, v^i_k]$, and the action vector $u^i_k$ be $[\omega^i_k,\alpha^i_k]$. For each agent $i$, we consider costs of the form \eqref{eq:multi_1_eq_1a} and \eqref{eq:multi_1_eq_1b} with conditions from Lemma~\ref{lemma-4}. Here $L^{ii}(\cdot)$ consists of the distance to goal cost and control costs while $L^{ij}(\cdot)$ consists of collision avoidance costs. We choose inter-agent collision avoidance costs to be of the form
\begin{align}
    L^{ij}_k(x^i_k,x^j_k) = \mathbf{1}\{d(x^i_k,x^j_k) < d_{m}\}\cdot\left(d(x^i_k,x^j_k) - d_{m}\right)^2,
\end{align}
where $\mathbf{1}\{\cdot\}$ is the indicator function, $d(x^i_k,x^j_k)$ is the Euclidean distance between the two agents, and $d_{m}$ is the maximum threshold distance between the agents after which the agents start incurring the collision avoidance cost. We choose the value of $d_m$ to be 2\emph{m}.
Since the game is a $w$-potential game, we can employ the result from Theorem \ref{thm-1}. We use the iLQR algorithm described in Section \ref{multi-agent-inter} to minimize the resulting single-agent optimal control problem associated with minimizing the potential function. 
We choose coefficients of inter-agent costs ($c^i$'s) such that Agent 1 incurs a high cost for being close to other agents. The resulting trajectories are plotted in Fig. \ref{fig:three_unicycle_potential}. Since Agent 1 has a higher coefficient of inter-agent cost compared to the other agents, it yields more to the other agents. It deviates more compared to other agents from its nominal path when moving towards its goal location. This scenario can be considered a highway scenario where two agents are more aggressive while Agent 1 moves cautiously around such aggressive drivers. This shows that our algorithm is able to handle such real-life asymmetric situations.



Furthermore, to study constrained settings, we consider four agents sitting on the corners of a square of length 3m such that their goal positions are the opposite diagonal ends of the square (see Fig.~\ref{fig:four_agents}). Each agent wants to reach its goal position while avoiding collisions with others. We consider discrete-time unicycle dynamics similar to \eqref{eq:unicycle_dynamics} to model each agent $i$.
Suppose the desired position of each agent $i$ is $x^i_f$. We consider the cost function of each agent $i$ to be \vspace{-0.17cm}
\begin{align}\label{eq:cost}
    & J^i(x,u) =  \frac{1}{2}\left(x^i_{T} - x_f^i\right)^\top Q_f^i\left(x^i_{T} - x_f^i\right) \nonumber\\ & + \sum_{k=0}^{T - 1}\frac{1}{2}\left(x^i_{k} - x^i_f\right)^\top Q^i\left(x^i_{k} - x^i_f\right) + \frac{1}{2}{u^i}_{k}^\top C^iu^i_{k},
\end{align}
where $Q^i$ and $C^i$ are diagonal matrices determining the deviation cost from the reference position and control costs, respectively. The matrix $Q_f^i$ determines the cost of deviating from the desired position at the final time instance. To enforce collision avoidance among the agents, we consider a minimum separation constraint between the agents. We require the distance between any pair of agents to be greater than or equal to some threshold distance $d_{\text{collision}}$. Therefore, we impose collision avoidance constraints of the following form between any pair of agents
\vspace{-0.15cm}
\begin{align}\label{collision_constraint}
     g_{ij}(x^i_k,x^j_k,k):= -d(x^i_k,x^j_k) + d_{\text{collision}} \leq 0,
\end{align}

\noindent where $d(x^i_k,y^i_k)$ is the distance between agents $i$ and $j$ at time step $k$. We consider the value of $d_{\text{collision}}$ to be 0.3\emph{m}. 
Furthermore, we consider additional constraints for bounding the actions of every each agent:
\begin{align}\label{bound_constraint}
    g_i(u^i_k,k) := |u^i_k| - u_{\text{bound}} \leq 0, 
\end{align}
\noindent where $u_{\text{bound}}$ is the bound on actions. We consider the bound $u_{\text{bound}}$ to be 3 m/s for the linear velocity and 3 rad/s for the angular velocity. All these constraints can be concatenated to obtain the constraint vector $g(x_k,u_k,k)$. It is straightforward to see that the game is an instance of a constrained potential dynamic game. Therefore, we apply the results from Theorem \ref{thm-1} to obtain optimal trajectories.
\begin{figure}
    \centering
    \includegraphics[scale=0.4]{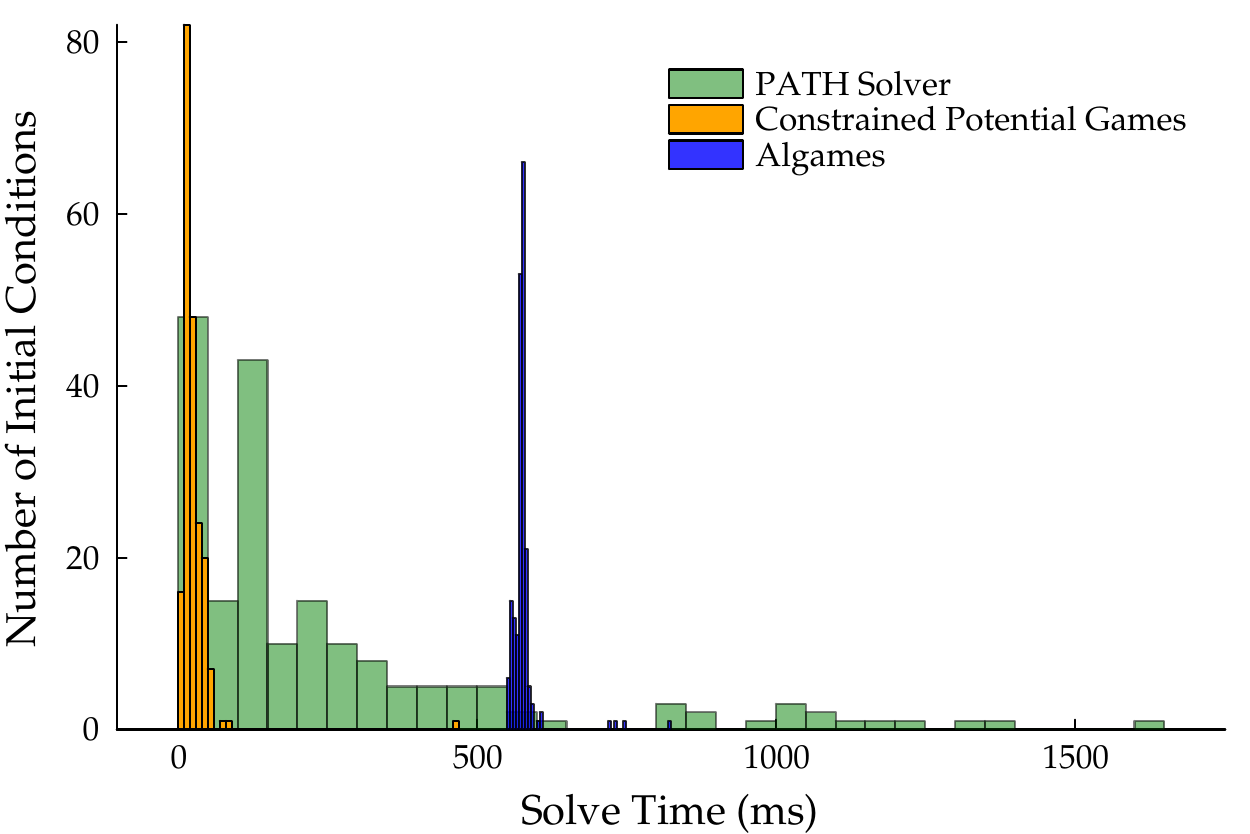}
    \caption{Histograms for the comparison of the solve time of Constrained Potential Games with Algames. Constrained Potential Games has an average solve time of 26.15 \emph{ms} in contrast to the PATH solver and Algames, which have an average solve time of 248.21 and 577.35 \emph{ms}, respectively. Our approach is approximately 10 times faster than the PATH solver and 20 times faster than Algames.}
    \label{fig:comparison}
    \vspace{-0.7cm}
\end{figure} 
The simulation of the resulting trajectories is demonstrated in~Fig.~\ref{fig:four_agents}, where the agents manage to successfully avoid collisions with one another while reaching their designated goal locations. They exhibit intuitive trajectories, yielding and coordinating their motions to avoid collisions. Similar trajectories can be generated for various initial conditions. We further compare the solve time of our algorithm with the state-of-the-art constrained game solvers such as Algames~\cite{cleac2019algames} and PATH solver~\cite{dirkse1995path}. In order to utilize the PATH solver, we use the method provided in~\cite{liu2023learning} to cast the game as a set of complementarity problems, which are then solved using the PATH solver. We generate 200 random sets of initial conditions for the four agents, and for each set of initial conditions, we run our algorithm, PATH solver, and Algames. Furthermore, we provide the same initializations for all the methods. We measure the solution time of both methods. The histograms of the solve time for the three methods are plotted in Fig.~\ref{fig:comparison}. The solve time of our method is on average $26.15 \pm 34.30$ \emph{ms} while the average solve time of Algames was $577.35 \pm 27.70$ \emph{ms}. 

While working with the PATH solver, we observed that sometimes it is difficult for the PATH solver to converge for highly symmetrical problems such as in Fig~\ref{fig:four_agents}. Because breaking symmetry becomes difficult without smart initializations when solving complementarity problems. Therefore, we observe that out of 200 initial conditions, the PATH solver converges for 189. The solve time is generally very high for the runs where the PATH solver doesn't converge. Therefore, we only consider the converged times for our solve time computations. The solve time for the PATH-solver is $248.21 \pm 302.05$ \emph{ms}.
As Fig.~\ref{fig:comparison} shows, our method is significantly faster than the PATH solver and Algames, and for most of the initial conditions, provides a solution in 20 \emph{ms} while for Algames, the solve time is always greater than 500 \emph{ms}. Furthermore, the variance in solve time for the PATH solver is really high. Our method is about 10 times faster than the PATH solver and 20 times faster than Algames.

\section{Experiments}\label{sec:experiment}
To demonstrate the performance of our proposed approach, we considered a hardware experiment where two Crazyflie quadrotors are navigating in a room shared with two humans. We require the quadrotors to transport a rigid rod with length $0.5$\emph{m} to enforce complicated task constraints. This implies that the quadrotors must maintain a fixed given distance from each other at all time steps, i.e., they must always satisfy an equality constraint. We further assume that agents are subject to collision avoidance constraints. This implies that quadrotors need to fulfill both equality and inequality constraints at each time step. The humans and quadrotors have designated initial and goal positions. The humans are asked to walk to their destinations, and the quadrotors need to carry the rod to the designated location while avoiding collisions with the humans.
The quadrotors use our trajectory optimization algorithm to move from their initial positions to goal positions. 
The state vector of each quadrotor $i \in \{1, 2\}$, at each time step $k$, is:
\begin{align*}
    x^i_k &= [p^i_{x, k}, ~p^i_{y, k}, ~p^i_{z, k}, ~\phi^i_{k}, ~\theta^i_k, ~\psi^i_k],
\end{align*}
\noindent which consists of its position ($[p^i_{x, k}, ~p^i_{y, k}, ~p^i_{z, k}]$) and orientation ($[\phi^i_{k}, ~\theta^i_k, ~\psi^i_k]$). We model each quadrotor as a 6 DOF integrator, i.e., each quadrotor dynamics are given by $\dot{x}^i_k = u^i_k$, where $u^i_k = [u^i_{1, k}, ~u^i_{2, k}, ~u^i_{3, k}, ~u^i_{4, k}, ~u^i_{5, k}, ~u^i_{6, k}]$ is the control input. The first three entries of $u^i_k$ correspond to the linear velocities, and the last three entries correspond to angular velocities. This design choice is based on the fact that we send waypoints as commands to Crazyflies through the Crazyswarm package \cite{crazyswarm}, and they track those waypoints with an in-place low-level controller.
Similar to~\cite{fridovich2020efficient}, we model the humans via unicycle dynamics, except for the fact that there is a constant height $r^i$ associated with each human $i, i \in \{1, 2\}$. We consider the human state vector to be
\begin{align*}
x^i_k = [p^i_{x, k}, ~p^i_{y, k}, ~r^i, ~\theta^i_{k}] 
\end{align*}
where $p^i_{x, k}, ~p^i_{y, k}$ denote the $x, y$ coordinates of human $i$ and $\theta^i_{k}$ denote the orientation of human $i$ (yaw), all at time step $k$.
We assume the $z$-coordinate of the human is located midway from the ground, i.e., $r^i \equiv 1$.
We consider the cost functions of quadrotors to be of the same form as in \eqref{eq:cost} where the reference trajectory for each agent is a straight line connecting their initial and goal locations. Furthermore, in the case of humans, we assume that the goal positions of humans are known through which we can identify the desired reference trajectory of the human. Based on that, we assume that the human's cost function is also the same as in \eqref{eq:cost}.\footnote{We acknowledge that the drones in our experiment may not have the perfect knowledge of the cost functions of humans. Therefore, typically we obtain the cost functions of other agents through inverse methods such as inverse reinforcement learning~\cite{russell1998learning} or inverse game theoretic methods~\cite{mehr2023maximum}.}
We assume that humans are rational, and while planning their motion in the environment, they also yield to each other and drones by satisfying constraints. We consider the collision avoidance constraint between the two humans to be the same as~\eqref{collision_constraint}. For the collision avoidance constraint between each quadrotor and each human, we assume that each human occupies a cylinder of height 2m with radius $\sqrt{0.4}\, m$, centered at $(p^i_{x, k}, ~p^i_{y, k}, ~r^i)$. 
The maximum linear velocity for the quadrotors is $1.2$m/s, and we assume that humans have a maximum linear velocity of $1.5$m/s.\\
We implement our method in a receding horizon fashion to account for the uncertainties in human actions and the potential mismatch between the humans' cost functions and the drones' assumption of the humans' cost functions to some extent. However, due to the virtue of the assumptions, our method will not work well in case the cost functions are drastically different from what the drones assume.
We set the planning horizon to be 0.5s with a step size of 0.1s. 
We solve the underlying optimization problem with do-mpc
\cite{LUCIA201751}, which models the problem symbolically with CasADi \cite{Andersson2019} and solves them with IPOPT \cite{Wachter:2006wt}.
\\ The resulting trajectories of our experiment are plotted in Fig \ref{fig:hardware_experiment}. 
As shown in the figure, when the two quadrotors are relatively far from the two humans, they move straight toward their designated positions. When humans are nearby, the quadrotors can simultaneously change their orientations and the orientation of the rod to avoid collisions with humans while carrying the rod and maintaining the equality constraints. 

\section{Conclusion}\label{sec:Conclusion}
In conclusion, this study addresses the challenges posed by planning and decision-making in multi-agent interactions. The computational complexity of reasoning about the actions and intentions of other agents, along with the diverse objectives of the agents involved, makes finding solutions a difficult task. By leveraging concepts from dynamic noncooperative game theory and constrained potential games, we proposed a novel approach to simplify the computation of OLGNE.
The key idea is to identify structures in realistic multi-agent planning scenarios that can be transformed into WCPDGs. Under specific cost structures, the problem can be formulated as a WCPDG, and the OLGNE can be obtained by solving a single constrained optimal control problem. This transformation significantly reduces the computational burden compared to solving a set of coupled constrained optimal control problems. We showcase the capabilities of our method through extensive numerical simulations and hardware experiments. We show that our method is significantly faster than state-of-the-art game solvers. 

In the future, we aim to expand the algorithm's capabilities to encompass more general cost functions, even accommodating different functional forms for inter-agent costs. We plan to extend the scope of our cost functions to include highly versatile structures, such as neural networks. By doing so, we intend to merge inverse cost learning methods with potential games, enabling the application of these ideas to real-world robotic systems.


\appendix

\subsection{Proof of Theorem \ref{thm-1}}\label{proof-thm-1}

\vspace{-1mm}\begin{proof}
This theorem was proven in \cite{zazo2016dynamic} for exact potential games with an infinite horizon when $T = \infty$. Here, we state the proof for WCPDG with finite horizons. It should be noted that as per \cite{zazo2016dynamic}, Assumption~\ref{assum-2} is required to ensure that the KKT conditions hold at the optimal point and that there exist feasible dual variables( see \cite{bertsekas1997nonlinear}, Prop. 3.3.8). With that, the Lagrangian for each agent $i \in \mathcal{N}$ can be written as
    
\begin{align}
    & \mathcal{L}(x_k,u_k,\lambda^i,\mu^i) = L^i(x_{T}) + {\mu^{i^\top}_{T}} g(x_{T}) + \sum_{k=0}^{T-1}\bigg( L^i_k(x_k,u_k) \bigg. \nonumber \\
    & \qquad \bigg. + {\lambda^{i^\top}_k} \left( f(x_k,u_k) - x_{k+1} \right) + {\mu^{i^\top}_k} g(x_k,u_k)\bigg),\end{align}
where $\lambda^i = \{\lambda^i_k\}_{k \in \{0,1,\ldots,T\}}$ and $\mu^i = \{\mu^i_k\}_{k \in \{0,1,\ldots,T\}}$ are the vector of Lagrange multiplies for agent $i$. For every agent $i\in\mathcal{N}$, the KKT conditions for the game $\mathcal{G}$ can be written as:
\begingroup
\allowdisplaybreaks
\begin{align}
    &\frac{\partial L^i_k(x_k,u_k)}{\partial x^i_k} + {\lambda^{i^\top}_k}\frac{\partial f(x_k,u_k)}{\partial x^i_k} + {\mu^{i^\top}_k}\frac{\partial g(x_k,u_k)}{\partial x^i_k} - \lambda^{i}_{k-1} = 0, \\
    & \frac{\partial L^i_k(x_k,u_k)}{\partial u^{i}_k} + {\lambda^{i^\top}_k}\frac{\partial f(x_k,u_k)}{\partial u^i_k} + {\mu^{i^\top}_k}\frac{\partial g(x_k,u_k,k)}{\partial u^i_k} = 0, \\
    & x_{k+1} = f(x_k,u_k), g(x_k,x_k) \leq 0 \; \forall k\in\{0,1,\ldots,T-1\};
    \\
    & \mu^i_k \leq 0, \;\; {\mu^{i^\top}_{k}} g(x_k,u_k) = 0, \; \forall \; 0\leq k \leq T-1. && \\
    & \frac{\partial L^i_T(x_{T})}{\partial x^i_{T}} + {\mu^{i^\top}_{T}}\frac{\partial g(x_{T})}{\partial x^i_{T}} = 0, && \\
    & g(x_{T}) \leq 0, \mu^i_T \leq 0, \;\; {\mu^{i^\top}_{T}} g(x_T) = 0. &&
\end{align}
\endgroup

\noindent To compute the KKT conditions for the multivariate constrained optimal control, we write the Lagrangian of \eqref{mopc} as
\begin{align}
    & \mathcal{L}^P(x_k,u_k,\xi_k) = L_T(x_{T})  + {\delta^\top_{T}} g(x_{T}) + \sum_{k=0}^{T-1}\bigg( L_k(x_k,u_k) \bigg. \nonumber \\
    & \qquad \bigg. + {\xi_k}^\top \left( f(x_k,u_k) - x_{k+1} \right) + {\delta^\top_k} g(x_k,u_k)\bigg)
\end{align}

\noindent where $\xi_k$ is the vector of Lagrange multipliers at time-step $k$. Therefore, its KKT conditions can be written as


\begin{align}
    & \frac{\partial L_k(x_k,u_k)}{\partial x^i_k} + {\xi^\top_k}\frac{\partial f(x_k,u_k)}{\partial x^i_k} + {\delta^\top_k}\frac{\partial g(x_k,u_k)}{\partial x^i_k} - \xi_{k-1} = 0, \\
    & \frac{\partial L_k(x_k,u_k)}{\partial u^{i}_k} + {\xi^\top_k}\frac{\partial f(x_k,u_k)}{\partial u^i_k} + {\delta^\top_k}\frac{\partial g(x_k,u_k,k)}{\partial u^i_k} = 0, \\
    & x_{k+1} = f(x_k,u_k,k), g(x_k,x_k) \leq 0 \; \forall k \in \{0,1,\ldots,T-1\} \\
    & \delta_k \leq 0, \;\; {\delta^\top_k} g(x_k,u_k) = 0, \; \forall \; 0\leq k \leq T-1. && \\
    & \frac{\partial L_T(x_{T})}{\partial x^i_{T}} +  {\delta^\top_{T}}\frac{\partial g(x_{T})}{\partial x^i_{T}}= 0 && \\
    & g(x_{T}) \leq 0, \delta_T \leq 0, \;\; {\delta^\top_T} g(x_T) = 0. &&
\end{align}
For each time step, we can multiply the KKT optimality conditions by a fixed positive number ($w^i_k$) on both sides and still the optimality conditions will remain equivalent to the original conditions. Therefore, the equivalent version of optimality conditions for \eqref{mopc} will be as follows:
\begin{align}
    & w^i_k\left(\frac{\partial L_k(x_k,u_k)}{\partial x^i_k} + {\xi^\top_k}\frac{\partial f(x_k,u_k)}{\partial x^i_k} \right. \\ \nonumber \\ & \qquad \qquad \qquad \left. + \; {\delta^\top_k}\frac{\partial g(x_k,u_k)}{\partial x^i_k} - \xi_{k-1}\right) = 0, \\
    & w^i_k\left(\frac{\partial L_k(x_k,u_k)}{\partial u^{i}_k} + {\xi^\top_k}\frac{\partial f(x_k,u_k)}{\partial u^i_k} + {\delta^\top_k}\frac{\partial g(x_k,u_k,k)}{\partial u^i_k}\right) = 0, \\
    & x_{k+1} = f(x_k,u_k,k), g(x_k,x_k) \leq 0 \; \forall k \in \{0,1,\ldots,T-1\} \\
    & \delta_k \leq 0, \;\; {\delta^\top_k} g(x_k,u_k) = 0, \; \forall \; 0\leq k \leq T-1. && \\
    & w^i_T\left(\frac{\partial L_T(x_{T})}{\partial x^i_{T}} +  {\delta^\top_{T}}\frac{\partial g(x_{T})}{\partial x^i_{T}}\right)= 0 && \\
    & g(x_{T}) \leq 0, \delta_T \leq 0, \;\; {\delta^\top_T} g(x_T) = 0. &&
\end{align}
Thus, in order for multivariate optimal control problem \eqref{mopc} and game $\mathcal{G}$ to have same optimality conditions the following must hold:
\begin{align}
\frac{\partial L^i_k(x^i_k,u_k)}{\partial x^i_k} &= w^i_k\left(\frac{\partial L_k(x_k,u_k)}{\partial x^i_k}\right), \label{eq:apend-1}\\
\frac{\partial L^i_k(x^i_k,u_k)}{\partial u_k^i} &= w^i_k\left(\frac{\partial L_k(x_k,u_k)}{\partial u_k^i}\right), \\
\forall i \in \mathcal{N}, \; k & = 0,1,\ldots,T-1; \; \text{and} \nonumber \\
\frac{\partial L^i_T(x^i_{T})}{\partial x^i_T} &= w^i_T\left(\frac{\partial L_T(x_{T},T)}{\partial x^i_T}\right), \label{eq:apend-2}\\
\lambda^i_k = w^i_k&\xi_k,\; \mu^i_k = \delta_k \; \forall i \in \mathcal{N}.
\end{align}
When conditions \eqref{eq:apend-1}-\eqref{eq:apend-2} are satisfied, Lemma-\ref{lemma-1} holds true. Since, Lemma-\ref{lemma-1} is necessary and sufficient condition for a game $\mathcal{G}:= \left( \mathcal{N}, \{\Gamma_i\}_{i\in\mathcal{N}}, \{J_i\}_{i\in\mathcal{N}},\{\mathcal{C}_k\}_{k\in\{0,\ldots, T\}},f,x_0 \right)$ to be a WCPDG, a solution to the \eqref{mopc} will be an open-loop Nash equilibria of $\mathcal{G}$. Readers are referred to proof of Theorem-1 in \cite{zazo2016dynamic} for further details.
\end{proof}


\subsection{Exempt Determination Involving Human Subjects}\label{sec: exempt}
This experiment was conducted while all the authors were working at the University of Illinois Urbana-Champaign(UIUC). The University of Illinois at Urbana-Champaign Office for the Protection of Research Subjects (OPRS) authorized the authors to use human subjects and determined that the criteria for exemption had been met. The protocol number for the same is 22107.

\bibliography{TRO}
\bibliographystyle{IEEEtran}

\begin{IEEEbiography}
[{\includegraphics[width=1in,height=1.25in,clip,keepaspectratio]{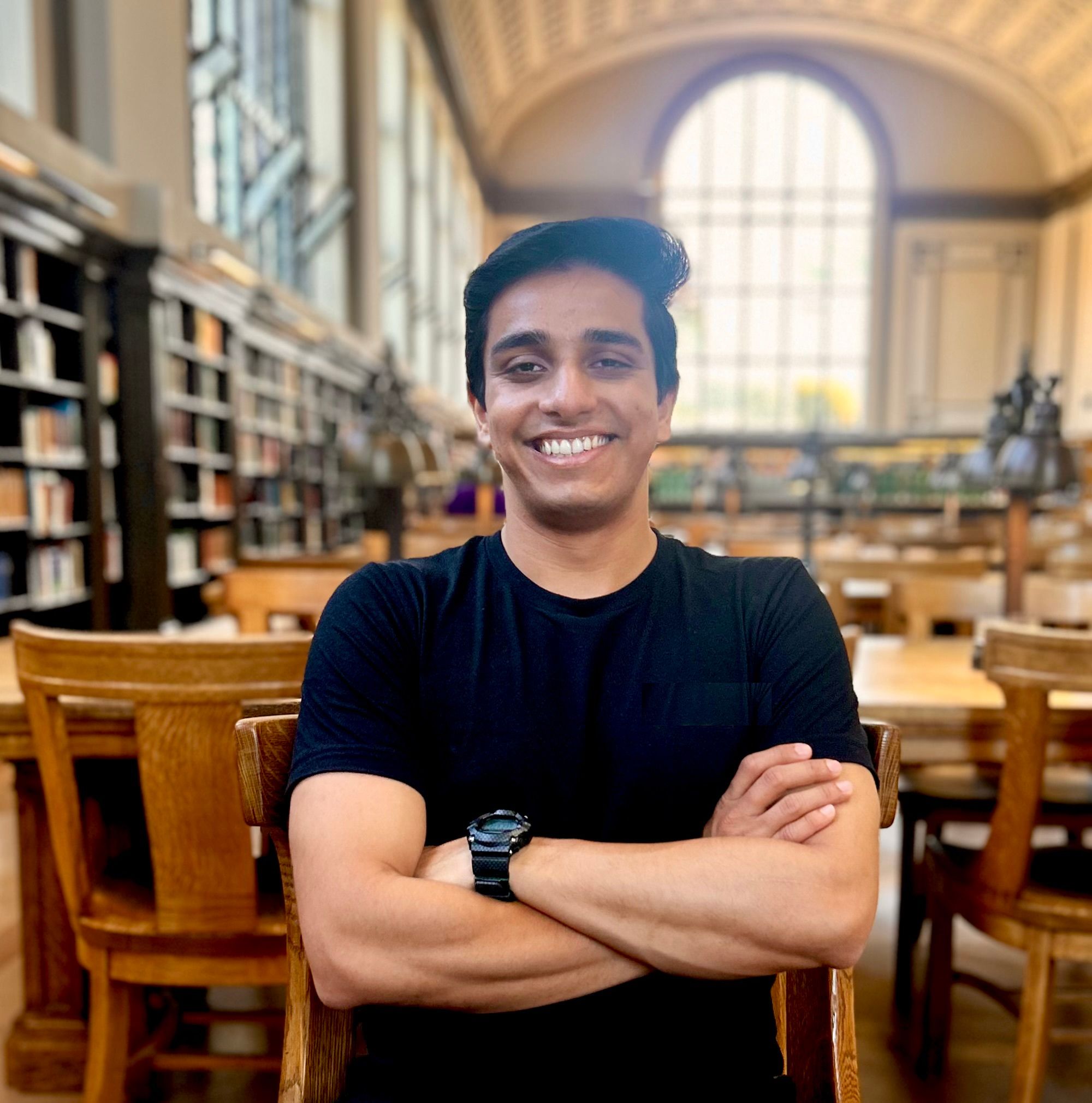}}]
{Maulik Bhatt} is a Ph.D. student in the Department of Mechanical Engineering at UC Berkeley. Before joining Berkeley, he was a Ph.D. student in the Department of Aerospace Engineering at the University of Illinois Urbana-Champaign. He received his Interdisciplinary Dual Degree (Bachelor's in Aerospace Engineering + Master's in Systems and Control Engineering) from the Indian Institute of Technology Bombay in 2021. He is interested in leveraging game theory, stochastic control, and machine learning to enable efficient robotic multi-agent interactions and safe motion planning. During his undergraduate studies, he worked in the area of state estimation using geometric control and variational principles.
\end{IEEEbiography}
\begin{IEEEbiography}
[{\includegraphics[width=1in,height=1.25in,clip,keepaspectratio]{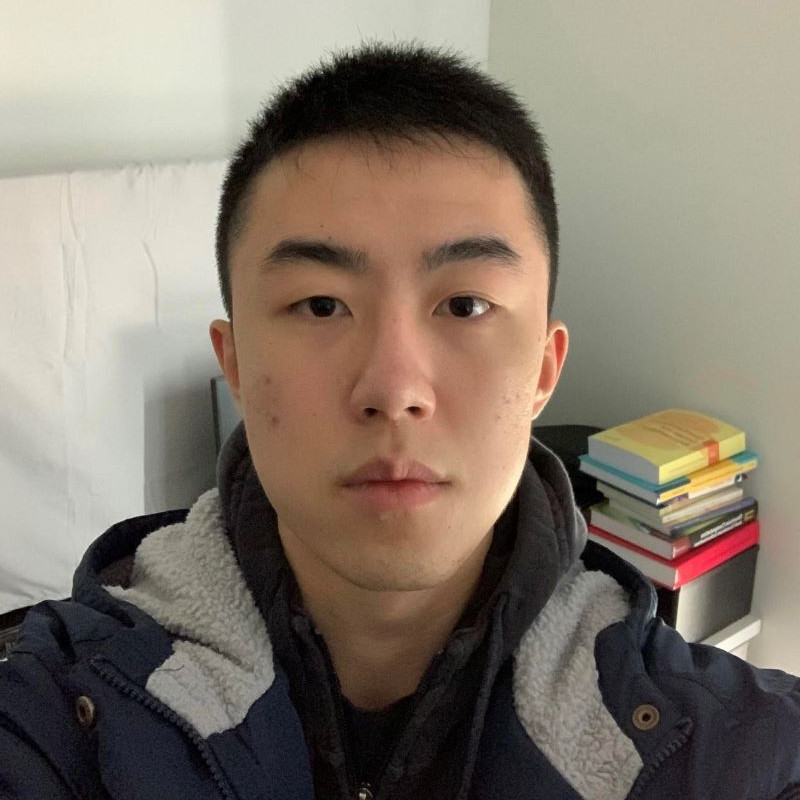}}]
{Yixuan Jia} received B.S. degrees in computer engineering and mathematics from the University of Illinois Urbana-Champaign, Champaign, IL, USA, in 2023. During his undergraduate studies, he studied multi-agent planning, safety verification of autonomous systems, and control theory.
He is currently working toward the M.S. degree in Aeronautics and Astronautics at the Massachusetts Institute of Technology.
His current research interests include perception, localization, and planning for autonomous navigation in unstructured environments. 
\end{IEEEbiography}
\begin{IEEEbiography}[{\includegraphics[width=1in,height=1.25in,clip,keepaspectratio]{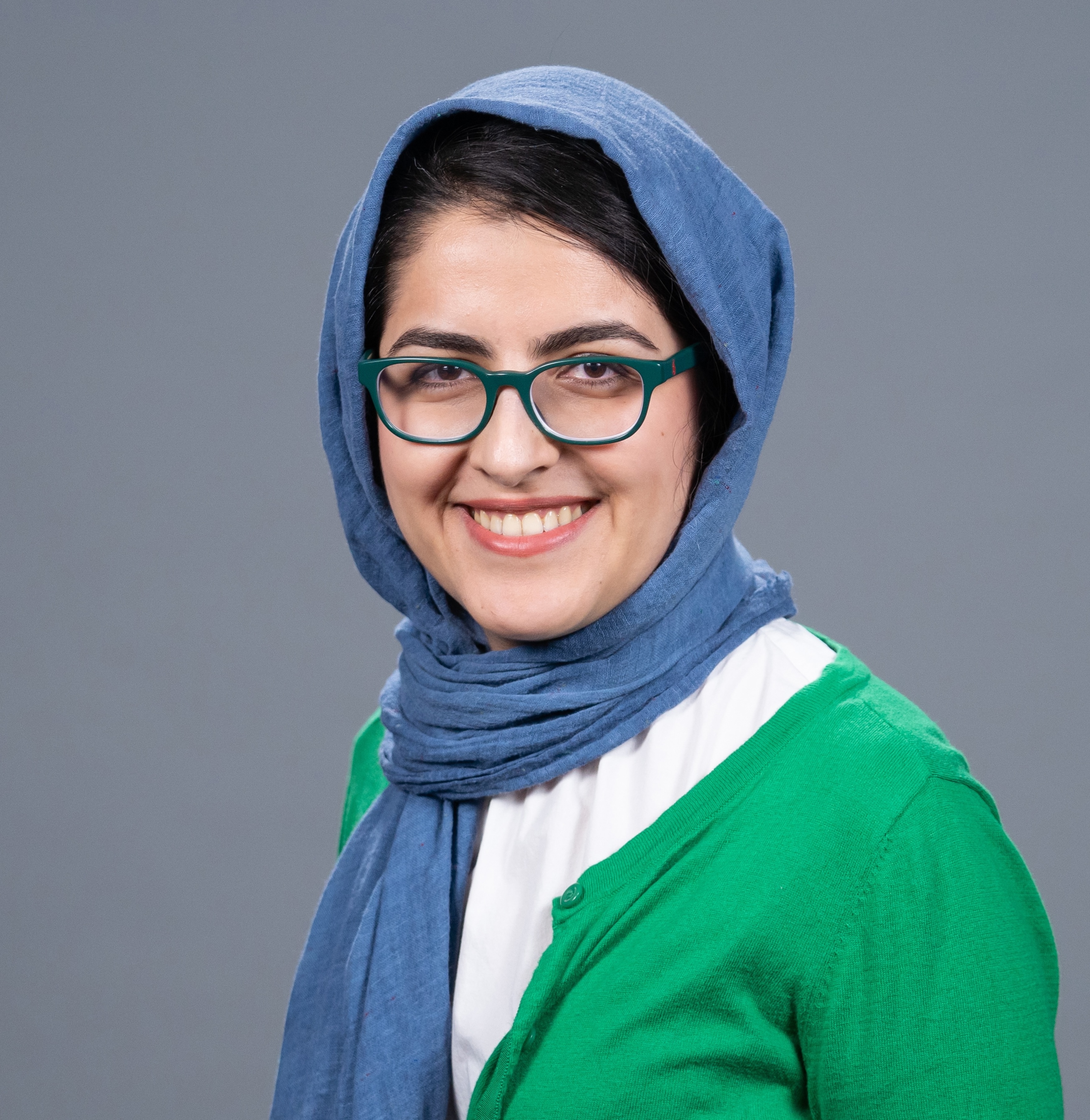}}]
    {Negar Mehr} (Member, IEEE) received a B.S. degree from the Sharif University of Technology and a Ph.D. degree from UC Berkeley, Berkeley, CA, USA, in 2013 and 2019, respectively, both in mechanical engineering. From 2019 to 2020, she was a Postdoctoral Scholar with Stanford Aeronautics and Astronautics Department. She is an Assistant Professor with the Mechanical Engineering Department at the University of California. Berkeley. Before joining Berkeley, she was an assistant professor in Aerospace Engineering at the University of Illinois Urbana-Champaign. Her research interest includes learning and control algorithms for interactive robots, robots that can safely and intelligently interact with other agents. Dr. Mehr is the recipient of the NSF CAREER award in 2022. Her Ph.D. dissertation was awarded the 2020 IEEE ITSS Best Ph.D. Dissertation Award.
\end{IEEEbiography}

\end{document}